\documentclass[10pt]{IEEEtran}
\usepackage{microtype}
\usepackage{graphicx}
\usepackage{subfigure}
\usepackage{booktabs} % for professional tables

\usepackage{hyperref}

\usepackage{amsmath}
\usepackage{amssymb}
\usepackage{mathtools}
\usepackage{amsthm}
\usepackage{comment}
\usepackage[capitalize,noabbrev]{cleveref}

\newtheorem{algo}{Algorithm}
\newtheorem{theorem}{{Theorem}}
\newtheorem{proposition}{{Proposition}}
\newtheorem{corollary}{{Corollary}}
\newtheorem{lemma}{{Lemma}}

\newtheorem{remark}{{Remark}}

\newtheorem{definition}{{Definition}}

\newtheorem{example}{{Example}}

\usepackage{enumerate}
\usepackage{url}
\usepackage{stfloats}
\usepackage{lineno}
\usepackage{hyperref}
\usepackage[capitalize,noabbrev]{cleveref}
\usepackage{bm}
\usepackage{bbm}
\usepackage{xcolor}
\usepackage{float}
\usepackage{tabularx}
\usepackage{tabu}
\usepackage{multicol}
\usepackage{multirow}
\usepackage{colortbl,booktabs,threeparttable}
\usepackage{dcolumn}
\usepackage{flushend}
\usepackage{soul}
\usepackage{ragged2e}
\usepackage{svg}

\graphicspath{{Figures/}}

\newcommand{\bmb}[1]{\bar{\bm{#1}}}
\newcommand{\bmh}[1]{\hat{\bm{#1}}}

\newcommand{\rv}[1]{\mathbf{#1}}
\newcommand{\rvu}[1]{\MakeUppercase{\bm #1}}
\newcommand{\rvb}[1]{\bar{\rv{#1}}}

\newcommand{\bb}[1]{\mathbb{#1}}
\renewcommand{\cal}[1]{\mathcal{#1}}
\newcommand{\bfit}[1]{\textbf{\textit{#1}}}

\newcommand{\Tr}{\operatorname{Tr}}

\newcommand{\R}{\mathbb{R}}
\newcommand{\E}{\mathbb{E}}
\newcommand{\D}{\mathbb{D}}
\renewcommand{\P}{\mathbb{P}}
\newcommand{\Ph}{\hat{\mathbb{P}}}
\newcommand{\Pnh}{\hat{\mathbb{P}}_n}
\newcommand{\Pnb}{\bar{\mathbb{P}}_n}
\newcommand{\Po}{\mathbb{P}_0}
\newcommand{\Pon}{\mathbb{P}^n_0}
\newcommand{\Pb}{\bar{\mathbb{P}}}

\renewcommand{\d}{\mathrm{d}}

\newcommand{\stp}{\hfill $\square$}

\newcommand{\ascvg}{\overset{a.s.}{\longrightarrow}}
\newcommand{\pcvg}{\overset{p}{\longrightarrow}}
\newcommand{\dcvg}{\overset{d}{\longrightarrow}}

\newcommand{\mub}{\bar \mu}

\newcommand{\grad}{\nabla}

\newcommand{\captext}[1]{\texorpdfstring{#1}{}}

\definecolor{hl-bg-color}{RGB}{255,255,215}
\sethlcolor{hl-bg-color} 
\definecolor{new-magenta}{RGB}{255,0,255}
\soulregister\citepp7
\soulregister\citeppp7
\soulregister\citeppt7
\soulregister\ref7

\newcommand{\tabincell}[2]{\begin{tabular}{@{}#1@{}}#2\end{tabular}}
\newcommand{\quotemark}[1]{``#1''}

\DeclareMathOperator*{\argmin}{argmin}

\usepackage{mathtools}
\newcommand{\defeq}{\coloneqq}

\begin{document}
\newpage
\title{Learning Against Distributional Uncertainty: On the Trade-off Between Robustness and Specificity}

\author{Shixiong Wang,~%,~\IEEEmembership{Member,~IEEE},
        Haowei Wang,
        Xinke Li,
        and Jean Honorio%,~\IEEEmembership{Senior~Member,~IEEE}% <-this % stops a space
\thanks{S. Wang is with the Institute of Data Science, National University of Singapore, Singapore 117602 (E-mail: s.wang@u.nus.edu). 
% S. Wang does not intend to update this affiliation; just keep it as is.

H. Wang is with Rice-Rick Digitalization, Singapore 179098 (E-mail: haowei\_wang@ricerick.com).

X. Li is with the Department of Data Science, City University of Hong Kong, Kowloon, Hong Kong. (E-mail:  xinkeli@cityu.edu.hk). 

J. Honorio is with School of Computing and Information Systems, The University of Melbourne, and
ARC Training Centre in Optimisation Technologies, Integrated Methodologies, and Applications (OPTIMA) (Email: jean.honorio@unimelb.edu.au).
}

\thanks{This research/project is supported by the National Research Foundation Singapore and DSO National
Laboratories under the AI Singapore Programme (AISG Award No: AISG2RP-2020-018).}

\thanks{\textit{Corresponding Author: Haowei Wang.}}
}

\maketitle

% abstract
\begin{abstract}
Trustworthy machine learning aims at combating distributional uncertainties in training data distributions compared to population distributions. 
Typical treatment frameworks include the Bayesian approach, (min-max) distributionally robust optimization (DRO), and regularization.
However, three issues have to be raised: 1) the prior distribution in the Bayesian method and the regularizer in the regularization method are difficult to specify; 2) the DRO method tends to be overly conservative; 3) all the three methods are biased estimators of the true optimal cost.
This paper studies a new framework that unifies the three approaches and addresses the three challenges above. 
The asymptotic properties (e.g., consistencies and asymptotic normalities), non-asymptotic properties (e.g., generalization bounds and unbiasedness), and solution methods of the proposed model are studied. The new model reveals the trade-off between the robustness to the unseen data and the specificity to the training data. Experiments on various real-world tasks validate the superiority of the proposed learning framework. 
\end{abstract}

% keywords

\section{Introduction} \label{sec:introdction}
Supervised statistical machine learning can be modeled by the following optimization problem \cite{kuhn2019wasserstein,shafieezadeh2019regularization}:
\begin{equation}\label{eq:true-opt}
	\min_{\bm x \in \cal X}~\E_{\rv \xi \sim \Po} h(\bm x, \rv \xi), 
\end{equation}
in which $\bm x \in \cal X \subseteq \R^l$ is the decision vector and $\rv \xi \in \Xi \subseteq \R^k$ is the random parameter whose underlying distribution is $\Po$; the cost function is denoted by $h: \cal X \times \Xi \to \R$ (particularly $\R_{+}$). Specifically, hypotheses are parameterized by $\bm x$ and $\rv \xi \defeq (\rv \xi_\text{in}, \rv \xi_\text{out})$ denotes a data pair where $\rv \xi_\text{in}$ and $\rv \xi_\text{out}$ denote the feature and expected response, respectively. 

In the practice of machine learning, the true population distribution $\Po$ is unknown, and the empirical distribution $\Pnh := \frac{1}{n} \sum^n_{i=1} \delta_{\rv \xi_i}$, where $\delta_{\rv \xi_i}$ is the Dirac distribution concentrated at the point $\rv \xi_i$, constructed by $n$ independent and identically distributed (i.i.d.) samples $\{\xi_i\}_{i\in[n]}$ is the most common estimate of $\Po$. As a result, we can use the data-driven \textbf{nominal model} \cite{kuhn2019wasserstein}
\begin{equation}\label{eq:erm-method}
	\min_{\bm x \in \cal X} \E_{\rv \xi \sim \Pnh}h(\bm x, \rv \xi)
\end{equation}
as an approximation to \textbf{true model} \eqref{eq:true-opt} to find the optimal decision. In the literature, \eqref{eq:erm-method} is known as an empirical risk minimization (ERM) model or a sample-average approximation (SAA) model. However, there exists a distributional mismatch (i.e., \textbf{distributional uncertainty}) between $\Pnh$ and $\Po$ due to scarce data and the approximation error of \eqref{eq:erm-method} to \eqref{eq:true-opt} vanishes only as $n \to \infty$. Neglecting such distributional uncertainty in $\Pnh$ may cause significant performance degradation: $\E_{\Po}h(\bmh x_n, \rv \xi)$ may be significantly larger than $\min_{\bm x} \E_{\Po}h(\bm x, \rv \xi)$ due to overfitting, where $\bmh x_n$ solves \eqref{eq:erm-method}. Mitigating the adverse impact resulting from the distributional uncertainty in $\Pnh$ and controlling the generalization error $\E_{\Po} h(\bm x^\star, \rv \xi) - \E_{\Pnh} h(\bm x^\star, \rv \xi)$ by selecting a promising decision $\bm x^\star$, in $\Pon$-probability or in $\Pon$-expectation, lie in the core of trustworthy machine learning, where $\Pon$ is the joint distribution of $n$ i.i.d. training samples.

\subsection{Literature Review}
Bayesian methods \cite{ferguson1973bayesian,ghosal2017fundamentals} are the first choice to deal with the distributional mismatch in $\Pnh$. Suppose $\cal C$ is a family of admissible distributions on the measurable space $(\Xi, \cal B_{\Xi})$ where $\cal B_{\Xi}$ denotes the Borel $\sigma$-algebra on $\Xi$; in the literature, $\cal C$ is also called an \textbf{ambiguity set}. For instance, in consideration of the nominal problem \eqref{eq:erm-method}, $\cal C$ can be defined as a closed distributional ball with center $\Pnh$ and radius $\epsilon_n$, that is, $\cal C := B_{\epsilon_n}(\Pnh)$. Bayesian approaches attempt to design a probability measure $\bb Q$ on $(\cal C, \cal B_{\cal C})$, where $\cal B_{\cal C}$ denotes the Borel $\sigma$-algebra on $\cal C$ \cite{gaudard1989sigma}, and the following Bayesian counterpart for the nominal problem \eqref{eq:erm-method} is solved:
\begin{equation}\label{eq:bayesian-method}
\min_{\bm x \in \cal X} \E_{\P \sim \bb Q} \E_{\rv \xi \sim \P} h(\bm x, \rv \xi).
\end{equation}
In this case, the true population distribution $\Po$ is expected to be included in $\cal C$ and an ideal $\bb Q$ should be the one that lets the distributions in $\cal C$ concentrate at $\Po$. Namely, $\Po$ is the element most likely to be sampled from $\cal C$ according to $\bb Q$. Under some mild technical conditions, we can find a point $\P' \in \cal C$ satisfying $\E_{\bb Q} \E_{\P} h(\bm x, \rv \xi) = \E_{\P'} h(\bm x, \rv \xi)$ for all $\bm x$ (see Lemma \ref{thm:bayeisan-mean-dist}). Hence, essentially, Bayesian methods tell us how to locate the \quotemark{best} candidate in $\cal C$. If $\P'$ is closer to $\Po$ than $\Pnh$ to $\Po$, the Bayesian method \eqref{eq:bayesian-method} would have a smaller approximation error for \eqref{eq:true-opt} than the nominal method \eqref{eq:erm-method} would have; examples and justifications can be accessed in, e.g., \cite{wu2018bayesian,anderson2020can}. Note that either (resp. both) $\bb Q$ or (resp. and) $\P$ can be parametric distributions.

Regularization approaches are another promising choice to hedge against the distributional uncertainty in $\Pnh$ \cite[Sec.~A1.3]{vapnik1998statistical}. To be specific, a regularization term $f(\bm x)$ is employed and the regularized counterpart
\begin{equation}\label{eq:regularization-method}
\min_{\bm x \in \cal X} \E_{\rv \xi \sim \Pnh} h(\bm x, \rv \xi) + \lambda f(\bm x)
\end{equation}
for the nominal empirical risk minimization problem \eqref{eq:erm-method} is studied, in which $\lambda \ge 0$ is a balancing coefficient. For example, the regularizer $f(\bm x)$ can be a proper norm $\|\bm x\|$ on $\cal X$ \cite[Chap.~7]{goodfellow2016deep}, and $\lambda$ may depend on the sample size $n$. Regularization methods are believed to be able to work against \quotemark{overfitting} and reduce generalization errors in a great number of learning problems; one may reminisce about the \quotemark{bias-variance trade-off} in the machine learning literature \cite[Sec.~2.9]{hastie2009elements}. The rationale of the regularization methods can also be quantitatively justified from many other perspectives such as the measure concentration inequalities \cite{zhang2021concentration}, the stability properties of the learning algorithms \cite[Sec.~5.2]{bousquet2002stability}, and the PAC-Bayesian learning \cite[Sec.~2]{germain2016pac}, to name a few. The key point is that the regularized SAA cost $\E_{\Pnh} h(\bm x, \rv \xi) + \lambda f(\bm x)$ can be an upper bound of the unknown true cost $\E_{\Po} h(\bm x, \rv \xi)$ for all $\bm x$, and therefore, by minimizing the regularized SAA cost, the unknown true cost can also be controlled. However, the SAA cost $\E_{\Pnh} h(\bm x, \rv \xi)$ solely cannot serve as an upper bound of the true cost. 

The (min-max) distributionally robust optimization (DRO) counterpart
\begin{equation}\label{eq:dro-method}
\min_{\bm x \in \cal X} \max_{\P \in \cal C} \E_{\rv \xi \sim \P} h(\bm x, \rv \xi)
\end{equation}
for the nominal model \eqref{eq:erm-method} is another potential approach to handle the distributional uncertainty in $\Pnh$ \cite{rahimian2022frameworks,kuhn2024distributionally,levy2020large}. If the distributional family $\cal C$ contains the true distribution $\Po$, then the inequality $\E_{\Po} h(\bm x, \rv \xi) \le \max_{\P \in \cal C} \E_{\P} h(\bm x, \rv \xi)$ holds for all $\bm x$, and therefore, by minimizing the robust cost $\max_{\P \in \cal C} \E_{\P} h(\bm x, \rv \xi)$, the unknown true cost can also be controlled; for more interpretations and justifications of the DRO method, see  \cite{kuhn2019wasserstein,kuhn2024distributionally}.
According to, e.g., \cite{yue2021linear,esfahani2018data}, we can find a point $\P'$ in $\cal C$ such that $\max_{\P \in \cal C} \E_{\P} h(\bm x, \rv \xi) = \E_{\P'} h(\bm x, \rv \xi)$, for all $\bm x$, if some mild technical conditions on the function $h$ can be satisfied. Therefore, as an alternative to the Bayesian approach \eqref{eq:bayesian-method}, the DRO approach \eqref{eq:dro-method} chooses the \quotemark{best} candidate $\P'$ in $\cal C$ from another perspective. 
However, in the practice of DRO methods, elegantly specifying the size parameter $\epsilon_n$ of the employed ambiguity set $\cal C := B_{\epsilon_n}(\Pnh)$ is not easy because the radius can be neither too large nor too small. A small radius cannot guarantee $\Po$ to be included in $\cal C$. Consequently, the worst-case cost $\max_{\P \in \cal C} \E_{\P} h(\bm x, \rv \xi)$ cannot provide an upper bound for the unknown true cost. Conversely, if the radius is too large, the DRO methods would become overly conservative and the upper bound of the true cost specified by $\max_{\P \in \cal C} \E_{\P} h(\bm x, \rv \xi)$ may be extremely loose. 
In the DRO literature, typical design methods for $\epsilon_n$ and their drawbacks are as follows.
\begin{enumerate}
	\item The measure concentration bounds in, e.g., \cite{fournier2015rate} and \cite{kuhn2019wasserstein}, are just theoretical results, far away from practical utilization, because the involved constants depend on the true underlying distributions, which are unknown. In addition, measure concentration bounds are not tight. Third, measure concentration bounds are dependent on the dimension of $\rv \xi$, and therefore, they may face the curse of dimensionality \cite{gao2022finite}.
	\item Practical methods such as cross-validation \cite[p. 156]{esfahani2018data} and bootstrap \cite[p. 158]{esfahani2018data} are reliable if and only if the data size $n$ is sufficiently large. When $n$ is small, they may not work well \cite{chernick2011bootstrap,varoquaux2018cross}.
	\item Statistical inference methods presented in \cite{blanchet2019robust,blanchet2021statistical} also require $n$ to be large because the optimality of the presented methods is established in the asymptotic sense (i.e., when $n \to \infty$).
\end{enumerate}
According to, e.g., \cite[Thm. 10]{kuhn2019wasserstein}, \cite{shafieezadeh2019regularization}, under some technical conditions, the DRO approach \eqref{eq:dro-method} amounts to a regularized empirical risk minimization method \eqref{eq:regularization-method}, which also advocates why the DRO approach \eqref{eq:dro-method} is able to combat overfitting and provide excellent generalization performance.

\subsection{Research Gaps and Motivations}\label{subsec:challenges}
It is practically uneasy to specify prior distribution $\bb Q$ in Bayesian method \eqref{eq:bayesian-method}, regularizer $f(\bm x)$ in regularization method \eqref{eq:regularization-method}, and radius $\epsilon_n$ of distributional ball $B_{\epsilon_n}(\Pnh)$ in DRO method \eqref{eq:dro-method}. The three quantities cannot be arbitrarily specified, otherwise, the performances of the three associated methods cannot be guaranteed. For example, as explained before, $\epsilon_n$ can be neither too large nor too small. Therefore, the first motivation of this work is to design a new framework that frees us from the elaborate selection of prior distribution $\bb Q$, regularizer $f(\bm x)$, and radius $\epsilon_n$.

In addition, the DRO approach, SAA approach, and regularized SAA approach are biased estimators of the true optimal objective value \eqref{eq:true-opt} when $n$ is finite; the biases only vanish asymptotically (i.e., as $n \to \infty$). Hence, the second motivation of this work is to design a new model that is able to be unbiased for finite $n$, which brings the asymptotic statistical property to finite-sample learning.

\subsection{Contributions}
The contributions of this paper can be summarized as follows.
\begin{enumerate}
	\item A new framework that can combat the distributional uncertainty in $\Pnh$ is designed; see Section \ref{sec:new-model}, and Models \eqref{eq:BDR-opt} and \eqref{eq:bdr-method}. The framework generalizes Bayesian method \eqref{eq:bayesian-method}, regularization method \eqref{eq:regularization-method}, and DRO method \eqref{eq:dro-method} and suggests the instructions in designing $\bb Q$ and $f(\bm x)$; see Remark \ref{rem:BDR-interpret}. In addition, the framework reveals the trade-off between the robustness to the unseen data (i.e., the adverse distributional uncertainty in $\Pnh$) and the specificity to the training data (i.e., the exploitable empirical information in $\Pnh$); see Remark \ref{rem:trade-off}. Moreover, the framework can diminish the conservatism, and therefore improve the performance, of the DRO method; see Theorem \ref{thm:gen-err}, Remark \ref{rem:BDR-better-DRO}, and Examples \ref{ex:gen-err} and \ref{ex:gen-err-lin-reg}. Statistical properties of the new learning model such as consistencies, asymptotic normalities, generalization bounds, and unbiasedness are established; see Theorems \ref{thm:asym-properties}, \ref{thm:gen-err}, and \ref{thm:unbiasedness}.
	
	\item The proposed new model is specifically studied under the $\phi$-divergence and Wasserstein distributional balls, and respective solution methods are derived; see Section \ref{subsec:solution-method}. In particular, the solutions disclose two important insights from the perspective of data augmentation (see Examples \ref{rem:weight-modification} and \ref{rem:data-augmentation}), which intuitively explain the flexibility of the proposed learning model.
\end{enumerate}

\section{Notations and Preliminaries}\label{sec:prelimilary}
Notations used in this paper are summarized in Appendix \ref{append:notations}. Necessary DRO theories are reviewed in Appendices \ref{subsec:distributional-balls} and \ref{subsec:wasserstein-reformulation}. Statistical concepts including Glivenko--Cantelli class, Donsker class, and Brownian bridge are presented in Appendix \ref{append:statistical-concepts}. 
In this section, we focus on a reformulation of the Bayesian model \eqref{eq:bayesian-method}. We start with the concept of mean distribution.
\begin{definition}[Mean Distribution]\label{def:bayeisan-mean-dist}
	A distribution $\Pb$ satisfying $\Pb(E) = \int_{\R} \P(E) \bb Q(\d \P(E)),~\forall E \in \cal B_{\Xi}$ is a mean distribution of $\bb P$ under $\bb Q$. \stp
\end{definition}
Namely, the mean distribution is a mixture of distributions in $\cal C$ with weights determined by $\bb Q$. 
To be specific, for an event $E$ in $\cal B_{\Xi}$, $\P(E)$ is a random variable taking values on $\R_{+}$ and its distribution is specified by $\bb Q$. This definition can transform  Bayesian model \eqref{eq:bayesian-method}. 
\begin{lemma}[\cite{wang2022robustness-gen-error}]\label{thm:bayeisan-mean-dist}
	If $\Pb$ is the mean {distribution} of $\bb P$ under $\bb Q$ and $\E_{\bb Q}\E_{\P} |h(\bm x, \rv \xi)| < \infty$, then $\E_{\bb Q}\E_{\P}h(\bm x, \rv \xi) = \E_{\Pb} h(\bm x, \rv \xi)$ for every $\bm x$. \stp
\end{lemma}

In terms of model \eqref{eq:bayesian-method}, the most popular choice for a non-parametric prior distribution $\bb Q$ of $\P$, in Bayesian nonparametrics, is the Dirichlet-process prior. Furthermore,  when the $n$-sample empirical distribution $\Pnh$ is considered, the posterior non-parametric distribution of $\P$ is still a Dirichlet process whose mean distribution is
$\frac{\alpha}{\alpha + n} \Ph + \frac{n}{\alpha + n} \Pnh$, 
where $\Ph$ is \textit{a priori} knowledge of $\Po$ and $\alpha \ge 0$ is employed to quantify the trust level towards $\Ph$ \cite{ferguson1973bayesian}, \cite[Chap.~3]{ghosal2017fundamentals}. Specifically, if we trust the prior $\Ph$ more than the empirical distribution $\Pnh$,  $\alpha$ should be large. 
When the Dirichlet-process prior is utilized, as a result of Lemma \ref{thm:bayeisan-mean-dist}, the Bayesian model \eqref{eq:bayesian-method} becomes
\begin{equation}\label{eq:dro-saa-obj}
	\min_{\bm x} \frac{\alpha}{\alpha + n} \E_{\Ph} h(\bm x, \rv \xi) + \frac{n}{\alpha + n} \E_{\Pnh} h(\bm x, \rv \xi).
\end{equation}
It can be generalized into
\begin{equation}\label{eq:dro-saa-obj-general}
	\min_{\bm x} \beta_n \E_{\Ph} h(\bm x, \rv \xi) + (1-\beta_n) \E_{\Pnh} h(\bm x, \rv \xi)
\end{equation}
where the weight $\beta_n \in [0,1]$ depends on sample size $n$; 
$\beta_n$ can be an arbitrary sequence satisfying $\beta_n \to 0$ as $n \to \infty$.
Model \eqref{eq:dro-saa-obj-general} serves as a foundation for the new machine learning framework that we propose subsequently.

\section{New Framework: Bayesian Distributionally Robust Learning}\label{sec:new-model}
In real-world operation, it is often difficult to specify an exact (non-parametric Bayesian) prior $\bb Q$ for a Bayesian model \eqref{eq:bayesian-method}. This motivates us to study the second-order min-max (or worst-case) Bayesian distributionally robust optimization counterpart for the nominal model \eqref{eq:erm-method}
\begin{equation}\label{eq:bayesian-method-second-order-min-max}
	\min_{\bm x} \max_{\bb Q} \E_{\P \sim \bb Q} \E_{\rv \xi \sim \P} h(\bm x, \rv \xi),
\end{equation}
which is a robustified version of the Bayesian model. In particular, model \eqref{eq:bayesian-method-second-order-min-max} is a combination of a Frequentist and a Bayesian method: The random measure $\P$ follows the second-order probability measure $\bb Q$, and therefore, in terms of $\P$, \eqref{eq:bayesian-method-second-order-min-max} is a Bayesian method; the admissible values of $\bb Q$ are only assumed to lie in an ambiguity set (which is not explicitly specified here), and therefore, in terms of $\bb Q$, \eqref{eq:bayesian-method-second-order-min-max} is a Frequentist method.

Inspired by \eqref{eq:bayesian-method-second-order-min-max},
we shall study the worst-case version of \eqref{eq:dro-saa-obj-general}:
\begin{equation}\label{eq:BDR-opt}
	\min_{\bm x \in \cal X} \beta_n \max_{\P \in B_{\epsilon}(\Ph)} \E_{\P} h(\bm x, \rv \xi) + (1-\beta_n) \E_{\Pnh} h(\bm x, \rv \xi).
\end{equation}
Note that the uncertainty in $\bb Q$ is reflected by the uncertainty in the prior estimate $\Ph$ because $\Pnh$ is completely determined given samples $\{\rv \xi_i\}_{i \in [n]}$. 

\begin{remark}[Interpretation of Model \eqref{eq:BDR-opt}]\label{rem:BDR-interpret}
	Model \eqref{eq:BDR-opt} is a Bayesian non-parametric model in terms of the data distribution $\P$ and also a Frequentist distributionally robust optimization model in terms of the distribution $\bb Q$ of the data distribution; cf. \eqref{eq:bayesian-method-second-order-min-max}. Since \eqref{eq:BDR-opt} is equivalent to
	$
	\min_{\bm x \in \cal X} \E_{\Pnh} h(\bm x, \rv \xi) + \frac{\beta_n}{1-\beta_n} \max_{\P \in B_{\epsilon}(\Ph)} \E_{\P} h(\bm x, \rv \xi),
	$
	by letting 
	$
	\lambda_n := \frac{\beta_n}{1-\beta_n}
	$
	and 
	$$
	f(\bm x) := \max_{\P \in B_{\epsilon}(\Ph)} \E_{\P} h(\bm x, \rv \xi),
	$$
    \eqref{eq:BDR-opt} can be rewritten as 
	$
	\min_{\bm x \in \cal X} \E_{\Pnh} h(\bm x, \rv \xi) + \lambda_n f(\bm x),
	$
	which is a regularized SAA model \eqref{eq:regularization-method}. Also, when $\beta_n := 1$, \eqref{eq:BDR-opt} reduces to a  DRO model \eqref{eq:dro-method}; when $\beta_n := 0$, \eqref{eq:BDR-opt} reduces to a  SAA model \eqref{eq:erm-method}. Hence, the new model \eqref{eq:BDR-opt} is a generalized model that unifies the SAA model \eqref{eq:erm-method}, the Bayesian model \eqref{eq:bayesian-method}, the regularized SAA model \eqref{eq:regularization-method}, and the DRO model \eqref{eq:dro-method}. The benefit is that \eqref{eq:BDR-opt} suggests how to design $\bb Q$ in the Bayesian method \eqref{eq:bayesian-method} and $f(\bm x)$ in the regularization method \eqref{eq:regularization-method}.
	\stp
\end{remark}

In practice, it is uneasy to specify $\Ph$. Alternatively, if the distributional ambiguity set is constructed around $\Pnh$ rather than $\Ph$, the model \eqref{eq:BDR-opt} becomes completely data-driven:
\begin{equation}\label{eq:bdr-method}
	\min_{\bm x \in \cal X} \beta_n \max_{\P \in B_{\epsilon_n}(\Pnh)} \E_{\P} h(\bm x, \rv \xi) + (1-\beta_n) \E_{\Pnh} h(\bm x, \rv \xi).
\end{equation}
This is a change-of-center trick for the employed distributional ambiguity set: Non-rigorously speaking, we are assuming $\Ph$ is contained in $B_{\epsilon_{n,1}}(\Pnh)$ and $\Pnh$ is contained in $B_{\epsilon_{n,2}}(\Ph)$ for some radii $\epsilon_{n,1}, \epsilon_{n,2} \ge 0$. We call \eqref{eq:bdr-method} a \textbf{Bayesian distributionally robust} (BDR) optimization.

\begin{remark}[Robustness-Specificity Trade-off]\label{rem:trade-off}
	Since the objective of \eqref{eq:bdr-method} balances the worst-case cost specified by DRO and the nominal cost specified by SAA, the new model \eqref{eq:bdr-method} reveals the trade-off between the robustness to the distributional uncertainty (i.e., unseen data) and the specificity to the empirical information (i.e., training data). \stp
\end{remark}

In the following, we use a linear regression example with Gaussian data distribution to intuitively explain the BDR learning framework. Consider the data generating distribution $\rv \xi \defeq [\rv \xi_{\text{in}}; \rv \xi_{\text{out}}] \sim N(\bm 0, \bm \Sigma_0)$ and the linear regression model $\rv \xi_{\text{out}} = \bm x^\top \rv \xi_{\text{in}} + e$ where $e \in \R$ denotes the regression residual. The true optimization problem $\min_{\bm x} [\bm x^\top, -1] \E_{\Po} \rv \xi \rv \xi^\top [\bm x; -1]$ admits 
\[
 \min_{\bm x \in \cal X} [\bm x^\top, -1]  \cdot  \bm \Sigma_0  \cdot  [\bm x; -1].
 \tag*{(True)}
\]
Denoting $\bmh \Sigma_n$ as the sample-estimate of $\bm \Sigma_0$, the SAA counterpart $\min_{\bm x} [\bm x^\top, -1] \E_{\Pnh} \rv \xi \rv \xi^\top [\bm x; -1]$ is 
\[
 \min_{\bm x \in \cal X} [\bm x^\top, -1]  \cdot  \bmh \Sigma_n  \cdot  [\bm x; -1].
 \tag*{(SAA)}
\]
The DRO counterpart $\min_{\bm x} \max_{\P} [\bm x^\top, -1] \E_{\P} \rv \xi \rv \xi^\top [\bm x; -1]$ under the order-$2$ Wasserstein ball $W_2(\P, \Pnh) \le \epsilon_n$ is
\[
\begin{array}{cl}
\min_{\bm x} \max_{\bm \Sigma} & [\bm x^\top, -1] \bm \Sigma [\bm x; -1] \\
\text{s.t.} & \Tr [\bm \Sigma + \bmh \Sigma_n - 2({\bm \Sigma}^{1/2} \bmh \Sigma_n {\bm \Sigma}^{1/2})^{1/2}] \le \epsilon_n^2,
\end{array}
\]
for which the von Neumann's minimax theorem holds. If $\bm \Sigma^*_n$ solves the above display ($\bm \Sigma^*_n$ may depend on $\bm x$), the DRO problem becomes
\[
 \min_{\bm x \in \cal X} [\bm x^\top, -1]  \cdot  \bm \Sigma^*_n  \cdot  [\bm x; -1].
\tag*{(DRO)}
\]
As a result, the BDR counterpart is
\[
\min_{\bm x \in \cal X} [\bm x^\top, -1] \cdot [\beta_n \bm \Sigma^*_n + (1 - \beta_n) \bmh \Sigma_n] \cdot [\bm x; -1].
\tag*{(BDR)}
\]

\section{Statistical Properties of  BDR Model \captext{\eqref{eq:bdr-method}}}\label{subsec:properties-BDR}
This subsection studies the asymptotic and non-asymptotic statistical properties of the new BDR model \eqref{eq:bdr-method} under any appropriate distributional ball $B_{\epsilon_n}(\Pnh)$, for example, the $\phi$-divergence ball or the Wasserstein ball, whose mathematical definitions can be found in Appendix \ref{subsec:distributional-balls}. 
Statistical concepts such as Glivenko--Cantelli class, Donsker class, and Brownian bridge can be found in Appendix \ref{append:statistical-concepts}; see also \cite[Chap.~19]{vdv1998asymptotic}. 
The key notations in this subsection are given in Table \ref{table:notations}.
\begin{table}[!htbp]
\small
\centering
\caption{Notation list. (``Opt. Sln.'' stands for Optimal Solution.)}
\begin{tabular}{c|c|p{0.20\textwidth}}
\hline
Notation & Definition & Mathematical Form \\
\hline
\(v(\bm{x})\) & True Cost & \(\E_{\Po} h(\bm{x}, \rv{\xi})\) \\
\(v_n(\bm{x})\) & SAA Cost & \(\E_{\Pnh} h(\bm{x}, \rv{\xi})\) \\
\(v_{r,n}(\bm{x})\) & DRO Cost & \(\max_{\P \in B_{\epsilon_n}(\Pnh)} \E_{\P} h(\bm{x}, \rv{\xi})\) \\
\(v_{b,n}(\bm{x})\) & BDR Cost & \(\beta_n \max_{\P \in B_{\epsilon_n}(\Pnh)} \E_{\P} h(\bm{x}, \rv{\xi})\) \\ & & ~~~~~~~~\(+ (1-\beta_n) \E_{\Pnh} h(\bm{x}, \rv{\xi})\) \\
\(\cal{X}_0\) & True Opt. Sln. Set & \(\argmin_{\bm{x} \in \cal{X}} v(\bm{x})\) \\
\(\hat{\cal{X}}_n\) & SAA Opt. Sln. Set & \(\argmin_{\bm{x} \in \cal{X}} v_n(\bm{x})\) \\
\(\hat{\cal{X}}_{r,n}\) & DRO Opt. Sln. Set & \(\argmin_{\bm{x} \in \cal{X}} v_{r,n}(\bm{x})\) \\
\(\hat{\cal{X}}_{b,n}\) & BDR Opt. Sln. Set & \(\argmin_{\bm{x} \in \cal{X}} v_{b,n}(\bm{x})\) \\ 
 $\bm{x}_0 $ & True Opt. Sln. & $\bm{x}_0 \in \cal{X}_0 $ \\
  $ \bmh{x}_n$ & SAA Opt. Sln.& $ \bmh{x}_n \in \hat{\cal{X}}_n$ \\
  $ \bmh{x}_{r,n} $ &DRO Opt. Sln.& $\bmh{x}_{r,n} \in \hat{\cal{X}}_{r,n}$ \\ 
  $ \bmh{x}_{b,n} $ & BDR Opt. Sln.& $  \bmh{x}_{b,n} \in \hat{\cal{X}}_{b,n}$ \\
\hline
\end{tabular}
\label{table:notations}
\end{table}

\subsection{Asymptotic Properties of \captext{\eqref{eq:bdr-method}}}
We consider the $\bm x$-parametric function class 
\begin{equation}\label{eq:class-H}
	\cal H \defeq \{h(\bm x, \cdot): \Xi \to \R | \bm x \in \cal X\}
\end{equation}
indexed by $\cal X$. The asymptotic properties of Bayesian distributionally robust model \eqref{eq:bdr-method} are given below, which illustrate the learning effectiveness when the sample size becomes infinitely large, as the generalization error approaches zero.

\begin{theorem}[Asymptotic Properties of \eqref{eq:bdr-method}]\label{thm:asym-properties}
	Consider the nominal problem \eqref{eq:erm-method} and its Bayesian distributionally robust counterpart \eqref{eq:bdr-method}.
	If the following conditions hold
	\begin{enumerate}[C1)]
		\item The DRO objective $v_{r,n}(\bm x)$ is bounded in $\Pon$-probability and attainable for $\bm x \in \cal X' \subseteq \cal X$;
		\item The weight coefficient $\beta_n \in [0, 1]$ for every $n$ and $\sqrt{n} \beta_n \to 0$ as $n \to \infty$;
		\item The function class $\cal H$ in \eqref{eq:class-H} is $\Po$-Glivenko--Cantelli;
		\item At least one of the following properties holds for the function $v(\bm x) = \E_{\Po}h(\bm x, \rv \xi)$:
		\begin{enumerate}[C4a)]
			\item $v(\bm x)$ is continuous on $\cal X'$;
			\item $v(\bm x)$ has the unique global minimizer $\bm x_0$ on $\cal X'$;% (i.e., $\cal X_0$ is a singleton);
		\end{enumerate}
		\item The function class $\cal H$ in \eqref{eq:class-H} is $\Po$-Donsker;
		\item $\E_{\Po}[h(\bmh x_n, \rv \xi) - h(\bm x_0, \rv \xi)]^2 \pcvg 0$ as $\bmh x_n \pcvg \bm x_0$,\footnote{The notations $\pcvg$ and $\dcvg$ mean the convergence in probability and distribution, respectively.}
	\end{enumerate}
	then the following statements are true.
	\begin{enumerate}[S1)]
		\item \textup{(Point-Wise Consistency of Objective Function.)}
		\label{thm:consistency-func}
		For every $\bm x \in \cal X'$, we have $v_{b,n}(\bm x) \pcvg v(\bm x)$ as $n \to \infty$.
		
		\item \textup{(Consistency of Optimal Value.)}\label{thm:consistency-v}
		For every $\bmh x_{b,n} \in \hat{\cal X}_{b,n} \subseteq \cal X'$ and every $\bm x_0 \in \cal X_0 \subseteq \cal X'$, we have $v_{b,n}(\bmh x_{b,n}) \pcvg v(\bm x_0)$ as $n \to \infty$. In other words, $\min_{\bm x} v_{b,n}(\bm x) \pcvg \min_{\bm x} v(\bm x)$ as $n \to \infty$.
		
		\item \textup{(Consistency of Optimal Solution.)}
		\label{thm:consistency-s}
		The limit point of any solution sequence $\{\bmh x_{b,n}\}$ of \eqref{eq:bdr-method} is a solution of the true problem \eqref{eq:true-opt} in $\Pon$-probability: $\Po^n\{\hat{\cal X}_{b,n} \subseteq \cal X_0\}  \to 1$ as $n \to \infty$.
		
		\item \textup{(Point-Wise Asymptotic Normality of Objective Function.)}
		\label{thm:asy-normality-func}
		For every $\bm x \subseteq \cal X'$, we have $\sqrt{n}[v_{b,n}(\bm x) - v(\bm x)] \dcvg N(0, V_{v,\bm x})$ as $n \to \infty$, where
		$
		V_{v,\bm x} \defeq \D_{\Po} h(\bm x, \rv \xi)
		$ denotes the variance of $h(\bm x, \rv \xi)$ under $\Po$.
		
		\item \textup{(Asymptotic Normality of Optimal Value.)}
		\label{thm:asy-normality-v}
		For every $\bmh x_{b,n} \in \hat{\cal X}_{b,n} \subseteq \cal X'$ and every $\bm x_0 \in \cal X_0 \subseteq \cal X'$, if $\bmh x_{b,n} \pcvg \bm x_0$, we have $\sqrt{n}[v_{b,n}(\bmh x_{b,n}) - v(\bm x_0)] \dcvg N(0, V_v)$ as $n \to \infty$, where
		$
		V_v \defeq \D_{\Po} h(\bm x_0, \rv \xi).
		$
	\end{enumerate}
\end{theorem}
\begin{proof}
	See Appendix D-A in the supplementary materials.
\end{proof}

\begin{remark}[Practicability of Conditions]
The conditions C1)-C6) stipulated in Theorem \ref{thm:asym-properties} are not restrictive, as they can be easily fulfilled in practice; concrete examples can be found in Appendix \ref{append:condition-examples-in-theorem-asym}.
\stp
\end{remark}

Note that in conducting minimization over $\bm x$, it is sufficient to only consider the subset $\cal X'$ where objective functions are finite-valued. Note also that when the DRO objective $v_{r, n}(\bm x)$ is finite at $\bm x$, the SAA objective $v_{n}(\bm x)$ and the true objective $v(\bm x)$ will be finite as well because $\Pnh$ and $\Po$ are included in $B_{\epsilon_n}(\Pnh)$ for sufficiently large $\epsilon_n$. The asymptotic normality of the optimal solution of the BDR model \eqref{eq:bdr-method}, which requires stronger and therefore more restrictive technical conditions, is deferred to Appendix D-B in the supplementary materials.

\subsection{Non-Asymptotic Properties of \captext{\eqref{eq:bdr-method}}}
First, we discuss the one-sided generalization bound, which is a crucial non-asymptotic property in machine learning. 

DRO learning has better generalization ability than traditional ERM learning because by reducing DRO cost $v_{r,n}(\bm x)$, true cost $v(\bm x)$ can also be diminished; however, ERM cost $v_n(\bm x)$ cannot upper bound $v(\bm x)$. Nevertheless, DRO learning is usually criticized for its conservatism. Specifically, to guarantee that the true distribution $\Po$ is included in the distributional ball, the radius $\epsilon_n$ of the ball should be sufficiently large (cf. Appendix \ref{subsec:wasserstein-distance}), which leads to that for every $\bm x$, the upper bound $v_{r,n}(\bm x)$ may be extremely loose. In what follows, we show that BDR model \eqref{eq:bdr-method} can be less conservative than the DRO model when the same distributional ball (with the same radius $\epsilon_n$) is shared.

\begin{theorem}[Generalization Bound of \eqref{eq:bdr-method}]\label{thm:gen-err}
	For every $\eta \in (0,1]$ and every $\beta_n \in [\beta^*_n, 1]$, if $\Po^n [\Po \in {B}_{\epsilon_n}(\Pnh)] \geq 1-\eta$, then the true cost $v(\bm x)$ is upper bounded, with $\Pon$-probability at least $1-\eta$, by the BDR cost $v_{b,n}(\bm x)$:
	\begin{equation}\label{eq:BDR-gen-bound}
	v(\bm x) \le \beta_n v_{r,n}(\bm x) + (1 - \beta_n) v_n(\bm x),~~~\forall \bm x \in \cal X,
    \end{equation}
    where the smallest (i.e., best) value $\beta^*_n$ of $\beta_n$ satisfying the above display is
    \begin{equation}\label{eq:BDR-gen-bound-beta}
        \beta^*_n \defeq \max\left\{ \max_{\bm x \in \cal X} ~ \frac{v(\bm x) - v_n(\bm x)}{v_{r,n}(\bm x) - v_n(\bm x)},~0\right\}
    \end{equation}
    which takes values on $[0, 1]$ and we assume that $0/0 = 0$; in addition, $\beta^*_n < 1$ if one of the following conditions holds:
    \begin{enumerate}[\hspace{1em}C1)]
        \item $v_{r, n}(\bm x) > v(\bm x)$ for every $\bm x \in \cal X$;
        \item $v_n(\bm x) = v(\bm x)$ for all $\bm x \in \cal X$ such that $v_{r, n}(\bm x) = v(\bm x)$.
    \end{enumerate}
\end{theorem}
\begin{proof}
    See Appendix D-C in the supplementary materials.
\end{proof}

\begin{remark}\label{rem:BDR-better-DRO}
In Theorem \ref{thm:gen-err}, the best value $\beta^*_n$ depends on the unknown true distribution $\Po$ [via the true cost function $v(\bm x)$], which cannot be obtained in practice. This is reminiscent of the practical limitation of the DRO theory where the best radius $\epsilon^*_n$ also depends on the unknown true distribution $\Po$; see Appendix \ref{subsec:wasserstein-distance}, especially \eqref{eq:wasserstein-concentration-epsilon}. Hence, both DRO and BDR require empirical parameter tuning in real-world operation. However, Theorem \ref{thm:gen-err} suggests that whenever DRO is empirically perfectly tuned, it is possible to further improve performance by tuning the BDR parameter $\beta_n$; recall that DRO and BDR share the same distributional ball (with the same $\epsilon_n$). \stp
\end{remark}

Theorem \ref{thm:gen-err} justifies the rationale of the BDR learning \eqref{eq:bdr-method} from the perspective of generalization theory. The BDR generalization bound $v_{b,n}(\bm x)$ in \eqref{eq:BDR-gen-bound} tightens the DRO generalization bound $v_{r,n}(\bm x)$ for every distributional ball $B_{\epsilon_n}(\Pnh)$ such that $\Po \in B_{\epsilon_n}(\Pnh)$ because $v_{r,n}(\bm x) \ge v_{n}(\bm x)$. To clarify further, suppose that $\epsilon^*_n$ is the smallest value of $\epsilon_n$ such that $\Po \in B_{\epsilon^*_n}(\Pnh)$. According to the DRO theory, the DRO cost $v_{r,n}(\bm x)$ with $\epsilon_n = \epsilon^*_n$ is the tightest upper bound for the true cost $v(\bm x)$. However, Theorem \ref{thm:gen-err} indicates that this DRO bound $v_{r,n}(\bm x)$ can be further refined to the BDR bound $v_{b,n}(\bm x)$ even when $\epsilon_n = \epsilon^*_n$. The refinement is non-trivial (i.e., $\beta^*_n < 1$) if one of the conditions in Theorem \ref{thm:gen-err} holds, which is the case, e.g., when $\Xi$ is a subspace of $\R^k$. To be specific, see \cite[Thm. 6.3]{esfahani2018data} and \cite{shafieezadeh2019regularization} for $\bar v_{r,n}(\bm x) > v(\bm x)$ when $\Xi \ne \R^k$, where $\bar v_{r,n}(\bm x)$ is a computational surrogate (i.e., finite-dimensional reformulation) of $v_{r,n}(\bm x)$. To avoid the conservatism of the DRO method, \cite{long2023robust} introduces an alternative modeling framework known as robust satisfying. However, \cite{long2023robust} is not rooted in DRO, and therefore, most existing DRO-based machine-learning methods cannot be directly upgraded.

Another concrete example for Theorem \ref{thm:gen-err} is as follows.

\begin{example}\label{ex:gen-err}
	According to \cite[Thm. 6.3]{esfahani2018data}, if the cost function $h$ is convex in $\rv \xi$ on $\Xi = \R^k$, the support set $\Xi$ of $\rv \xi$ is a closed and convex set, the order $p$ of the Wasserstein distance is set to $p := 1$, and the employed metric $d$ in the Wasserstein distance is specified by a proper norm $\|\cdot\|$ on $\Xi$, then the distributionally robust optimization objective exactly equals to a regularized SAA objective, point-wisely for every $\bm x \in \R^l$: i.e., for every $\bm x \in \R^l$, we have
	$$
        v_{r,n}(\bm x) := \max_{\P: W_p(\P, \Pnh) \le \epsilon_n} \E_{\P} h(\bm x, \rv \xi) 
			= v_n(\bm x) + \epsilon_n \cdot f(\bm x)
    $$
	where
	$
	f(\bm x) \defeq \max_{\bm \theta \in \Xi}\{\|\bm \theta\|_*: h^*(\bm x, \bm \theta) < \infty\}
	$
	is a regularization term,\footnote{Similar results are reported in, e.g., \cite{blanchet2019robust,gao2020wasserstein,gao2022finite,shafieezadeh2019regularization}, where $f(\bm x)$ may be of different forms.} $\|\cdot\|_*$ denotes the dual norm of $\|\cdot\|$, and $h^*(\bm x, \bm \theta)$ denotes the Fenchel convex conjugate of $h(\bm x, \rv \xi)$ point-wisely for every given $\bm x \in \R^l$. 
    As a result, the generalization bound specified by the DRO model is 
	$$
    v(\bm x) \le v_{r,n}(\bm x) = v_n(\bm x) + \epsilon_n f(\bm x),~~~\forall \bm x \in \R^l.
    $$
    However, Theorem \ref{thm:gen-err} supports that the bound above can be tightened to
	$$
    v(\bm x) \le v_{b,n}(\bm x) \le v_n(\bm x) + \beta_n \epsilon_n f(\bm x),~ \forall \bm x \in \R^l,~ \exists \beta_n \in [0, 1].
    $$
    The best value of $\beta_n$ is 
    $
        \beta^*_n \defeq \max_{\bm x \in \cal X} \frac{v(\bm x) - v_n(\bm x)}{\epsilon_n f(\bm x)} \le 1.
    $
    The inequality is strict if 1) the radius $\epsilon_n$ is large; or 2) $\cal X$ is a specified subspace of $\R^l$ on which $v(\bm x) < v_n(\bm x) + \epsilon_n f(\bm x)$. 
    One may interpret $\beta_n \epsilon_n$ as the radius of a new distributional ball that may not include $\Po$ in the DRO sense. However, the true cost can still be upper-bounded, indicating that the conventional DRO bound is not sufficiently tight on the focused region $\cal X$, although it may be tight on the whole space $\R^l$.
	\stp
\end{example}

A specific instance of Example \ref{ex:gen-err} is given below.
\begin{example}[$1$-norm Linear Regression]\label{ex:gen-err-lin-reg}
    Let the data vector be $\rv \xi \defeq [\rv \xi_{\text{in}}; \rv \xi_{\text{out}}]$ and the true data generating model be $\rv \xi_{\text{out}} = \bm x^\top_0 \rv \xi_{\text{in}} + e$, where $\rv \xi_{\text{in}} \sim N(\bm 0, \bm E_{k-1})$ denotes the feature vector, $\bm E_{k-1}$ denotes the $(k-1)$-dimensional identity matrix, the standard Gaussian variable $e \in \R$ denotes the regression residual (uncorrelated with $\rv \xi_{\text{in}}$), and $\rv \xi_{\text{out}} \in \R$ denotes the response. 
    Consider the $1$-norm linear regression problem. Supposing that $\rv \xi \sim \Po$, we have 
    $$
    v(\bm x) = \E_{\Po} |\rv \xi_{\text{out}} - \bm x^\top \rv \xi_{\text{in}}|,
    $$
    $$
    v_n(\bm x) = \E_{\Pnh} |\rv \xi_{\text{out}} - \bm x^\top \rv \xi_{\text{in}}|,
    $$
    and according to Example \ref{ex:gen-err} and \cite[Eq.~(4.5)]{chen2020distributionally},
    $$
    \begin{array}{cl}
    v_{r,n}(\bm x) &= \max_{\P \in B_{\epsilon_n}(\Pnh)} \E_{\P} |\rv \xi_{\text{out}} - \bm x^\top \rv \xi_{\text{in}}| \\
    &= \E_{\Pnh} |\rv \xi_{\text{out}} - \bm x^\top \rv \xi_{\text{in}}| + \epsilon_n  \|(-\bm x, 1)\|_*. \\
    \end{array}
    $$
    As a demonstration, we particular $\|\cdot\|_*$ into the vector $2$-norm. Therefore, the best value $\beta^*_n$ is 
    \[
        \beta^*_n = \max_{\bm x} \frac{\E_{\Po} |\rv \xi_{\text{out}} - \bm x^\top \rv \xi_{\text{in}}| - \E_{\Pnh} |\rv \xi_{\text{out}} - \bm x^\top \rv \xi_{\text{in}}|}{\epsilon_n \|(-\bm x, 1)\|_2}.
    \]
    To visualize, we examine a one-dimensional case. We set the true parameter be $x_0 = 1$, the sample size $n = 1$, and the radius $\epsilon_n = 1$. Under one realization of $\Pnh$, the true, SAA, DRO, and BDR costs are shown in Fig. \ref{fig:BDR-example-beta}, where we assume that $x$ takes grid values on $[-4, 6]$ with step size of $0.01$. As we can see, if the feasible region of the decision variable $x$ is required to be $[-1.7, 1.7]$, the BDR bound in Fig. \ref{fig:BDR-example-beta-a} is no longer tight but the BDR bound in Fig. \ref{fig:BDR-example-beta-b} becomes tight. \stp
    
    \begin{figure}[!htbp]
    	\centering
    	\subfigure[$\beta = 0.50956$ (best)]{
    		\includegraphics[height=3.cm]{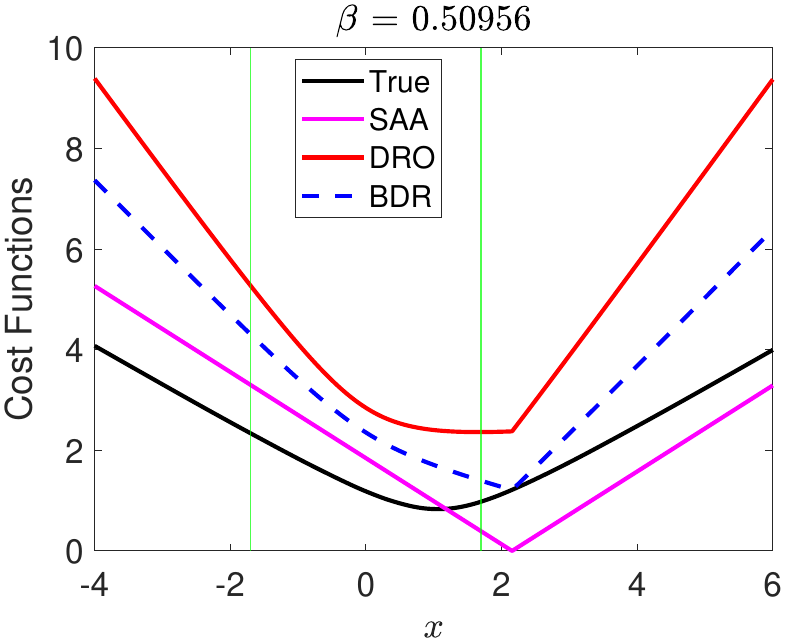}
            \label{fig:BDR-example-beta-a}
    	}

    	\subfigure[$\beta = 0.3$]{
    		\includegraphics[height=3.cm]{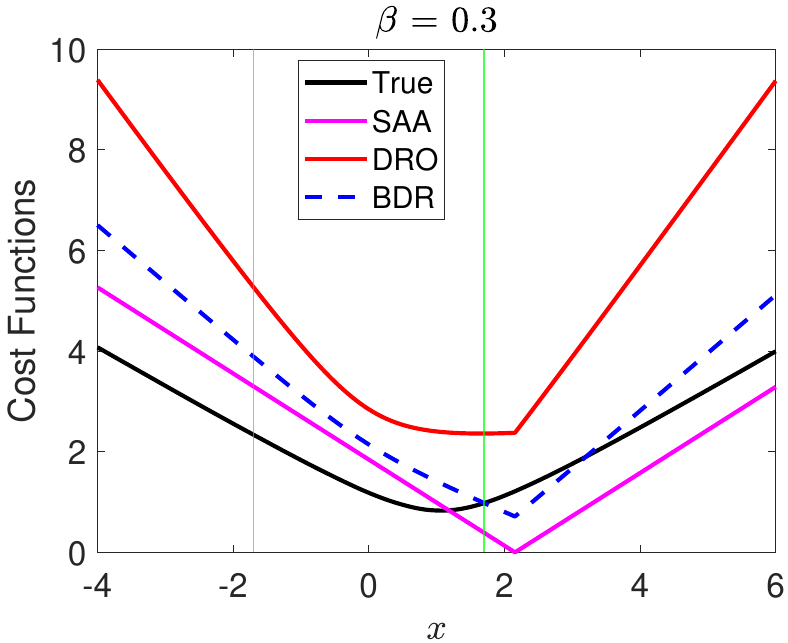}
        \label{fig:BDR-example-beta-b}
    	}
    	\subfigure[$\beta = 0.7$]{
    		\includegraphics[height=3.cm]{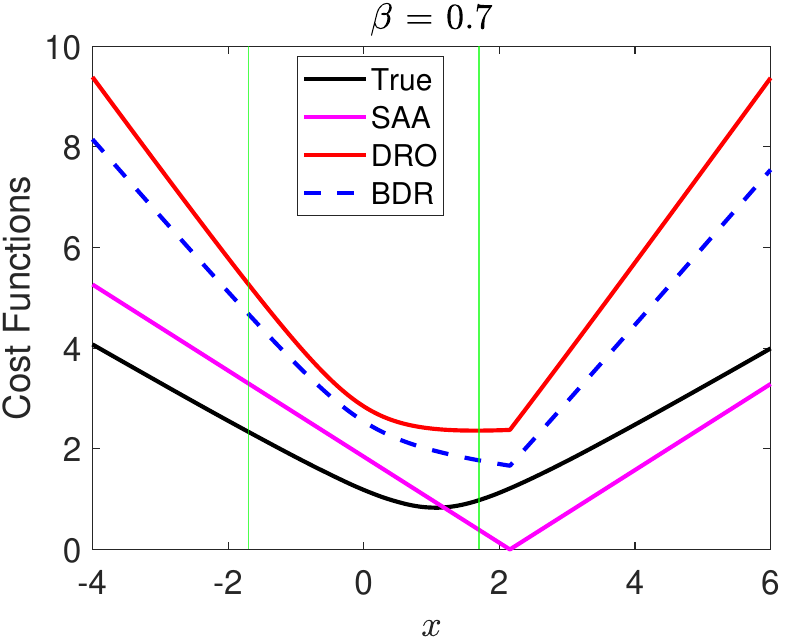}
    	}
    	\caption{Cost functions; the SAA cost cannot upper bound the true cost. (a): when $\beta = 0.50956$, the BDR cost function provides a tight upper bound for the true cost function; (b): when $\beta < 0.50956$, the BDR cost function cannot upper bound the true cost function; (c): when $\beta > 0.50956$, the BDR cost function provides a loose upper bound for the true cost function. If the feasible region of the decision variable $x$ is required to be $[-1.7, 1.7]$ rather than $\R$, the BDR bound in (a) is no longer tight but that in (b) becomes tight. (Source Codes: \url{https://github.com/Spratm-Asleaf/Robustness-Specificity}.)}
    	\label{fig:BDR-example-beta}
    \end{figure}
\end{example}

As a result of Theorem \ref{thm:gen-err}, focusing on the Bayesian distributionally robust solution $\bmh x_{b, n}$, the true cost $v(\bmh x_{b,n})$ of the BDR model is upper bounded, with $\Pon$-probability at least $1-\eta$, as 
$%\[
    v(\bmh x_{b,n}) \le \beta_n v_{r,n}(\bmh x_{b,n}) + (1 - \beta_n) v_n(\bmh x_{b,n}).
$%\]

Next, we discuss the unbiasedness of the BDR model.
The DRO model is always an upward (i.e., positively) biased estimator of the true optimal cost for all radius $\epsilon_n \ge 0$ such that $\Po \in B_{\epsilon_n}(\Pnh)$, while the SAA model is always a downward (i.e., negatively) biased estimator.\footnote{For technical details, see the proof of Theorem \ref{thm:unbiasedness}.} However, 
the BDR model can be unbiased with a proper $\beta_n$.

\begin{theorem}[Unbiasedness]\label{thm:unbiasedness}
	For every $n$, there exists $\beta_n \in [0,1]$ such that the BDR-estimated cost $v_{b,n}(\bmh x_{b,n})$ is an unbiased estimate of the true optimal cost $v(\bm x_0)$. 
\end{theorem}
\begin{proof}[Proof (sketch)]
	We first show that the DRO model is an upward (positively) biased model and the SAA model is a downward (negatively) biased model. Then, the BDR model is proved to be unbiased. For details, see Appendix D-D in the supplementary materials.
\end{proof}

The BDR model's unbiasedness indicates that achieving asymptotic statistical property is possible in finite-sample learning; note that this result is theoretically impossible for DRO and SAA models. However, a $\beta_n$ satisfying Theorem \ref{thm:gen-err} [i.e., \eqref{eq:BDR-gen-bound}] does not necessarily satisfy Theorem \ref{thm:unbiasedness} for unbiasedness, and vice versa. The finite-sample unbiasedness shows the statistical superiority of BDR over SAA and DRO.

\section{Solution Method of BDR Model \captext{\eqref{eq:bdr-method}}}\label{subsec:solution-method}
To solve BDR model \eqref{eq:bdr-method}, the key is to reformulate the DRO sub-problem $\max_{\P \in B_{\epsilon_n}(\Pnh)} \E_{\P} h(\bm x, \rv \xi)$ under a specified distributional ball $B_{\epsilon_n}(\Pnh)$. This paper examines the $\phi$-divergence and Wasserstein distributional balls; see Appendix \ref{subsec:distributional-balls}.

\subsection{\captext{$\phi$}-Divergence} 
We start with the $\phi$-divergence ball whose mathematical definition is available in Appendix \ref{subsec:phi-divergence}; this case is practical if the underlying true data-generating distribution $\Po$ is discrete.

\begin{theorem}\label{thm:DRO-reformulation-phi}
	Consider the $\phi$-divergence distributional ball $B_{\epsilon_n, \phi} (\Pnh)$ induced by the $\phi$-divergence. The DRO sub-problem $\max_{\P \in B_{\epsilon_n, \phi} (\Pnh)} \E_{\P} h(\bm x, \rv \xi)$ can be reformulated to
	\begin{equation}\label{eq:dro-model-discrete-phi-special}
			\displaystyle \max_{ \bm \mu \in \R^n} \textstyle \sum^n_{i = 1} \mu_i \cdot h(\bm x, \rv \xi_i),~~~~~			\text{s.t.}~~~~~F_{\phi} (\bm \mu \| \bm \mub) \le \epsilon_n,
	\end{equation}
where $F_{\phi} (\bm \mu \| \bm \mub)$ defines the $\phi$-divergence of the discrete distribution $\bm \mu \defeq [\mu_1, \mu_2, \ldots, \mu_n]$ from the nominal distribution $\bmb \mu \defeq [1/n, 1/n, \ldots, 1/n]$; note that $\bm \mu, \bmb \mu \in \R^n$.
\end{theorem}
\begin{proof}
From Appendix \ref{subsec:phi-divergence}, we know that distributions $\P$ in $B_{\epsilon_n, \phi} (\Pnh)$ have the same support as $\Pnh$. Hence, distributions in $B_{\epsilon_n, \phi} (\Pnh)$ can be characterized as $\P = \sum^n_{i=1} \mu_i \delta_{\rv \xi_i}$, which completes the proof.
\end{proof}

A concrete example of the constraint in \eqref{eq:dro-model-discrete-phi-special} can be obtained using the Kullback--Leibler (KL) divergence: that is, 
{
\[
F_{\phi} (\bm \mu \| \bm \mub) \defeq \sum^n_{i = 1} \mu_i \cdot \log(\mu_i / \mub_i) = \sum^n_{i = 1} \mu_i \cdot \log(n \mu_i) \le \epsilon_n. 
\]}As a result, the solution of BDR method \eqref{eq:bdr-method} is given in the corollary below.

\begin{corollary}[Solution of BDR Method \eqref{eq:bdr-method} Under $\phi$-Divergence Ball]\label{cor:solution-BDR-phi}
The BDR model \eqref{eq:bdr-method} under the $\phi$-divergence ball can be reformulated into
	\begin{equation}\label{eq:solution-BDR-phi}
    {
		\begin{array}{cll}
			\displaystyle \min_{\bm x \in \cal X}  & \displaystyle \beta_n \max_{\bm \mu \in \R^n} \sum^n_{i = 1} \mu_i \cdot h(\bm x, \rv \xi_i) + (1-\beta_n) \displaystyle \sum^n_{i = 1} \frac{1}{n} \cdot h(\bm x, \rv \xi_i) \\
			\text{s.t.} & F_{\phi} (\bm \mu \| \bm \mub) \le \epsilon_n,
		\end{array} 
   }
	\end{equation}
which is a finite-dimensional optimization. \stp
\end{corollary}

\subsection{Wasserstein Distance} 
We then study the Wasserstein distributional ball whose mathematical definition is available in Appendix \ref{subsec:wasserstein-distance}.

\begin{theorem}\label{thm:DRO-reformulation-wasserstein}
	Consider the Wasserstein distributional ball $B_{\epsilon_n, p} (\Pnh)$ induced by the order-$p$ Wasserstein distance. Suppose one of the following conditions holds: 1) For every $\bm x$, $h(\bm x, \rv \xi)$ is continuous in $\rv \xi$ on $\Xi$; 2) For every $\bm x$, $h(\bm x, \rv \xi)$ is concave in $\rv \xi$ on $\Xi$. Then, the DRO sub-problem $\max_{\P \in B_{\epsilon_n, p} (\Pnh)} \E_{\P} h(\bm x, \rv \xi)$ can be reformulated to
	\begin{equation}\label{eq:dro-model-discrete-wasserstein-special}
			\displaystyle \max_{\{\rv \zeta_j\}_{j \in [n]}} \textstyle \frac{1}{n} \sum^n_{j = 1} h(\bm x, \rv \zeta_j),~~~~\text{s.t.}~~~  \textstyle \frac{1}{n} \sum^n_{j = 1} d^p(\xi_j , \zeta_j) \le \epsilon^p_n,
	\end{equation}
 where $d$ is a distance on $\Xi$.
\end{theorem}
\begin{proof}
	See Appendix E in the supplementary materials. %The proof is based on the Monte--Carlo approximation, the law of large numbers, and the linear programming theory; this is a new proof scheme to the DRO community.
\end{proof}

A concrete example of the constraint in \eqref{eq:dro-model-discrete-wasserstein-special} can be obtained using the $2$-norm on $\Xi$ and $p \defeq 1$, that is, 
\[
\textstyle \frac{1}{n} \sum^n_{j = 1} \|\xi_j - \zeta_j\|_2 \le \epsilon_n.
\]
As a result, the solution of the BDR method \eqref{eq:bdr-method} is given in the corollary below.

\begin{corollary}[Solution of BDR Method \eqref{eq:bdr-method} Under Wasserstein Ball]\label{cor:solution-BDR-wasserstein}
The BDR model \eqref{eq:bdr-method} under the Wasserstein ball can be reformulated into
	\begin{equation}\label{eq:solution-BDR-wasserstein}
 {
		\begin{array}{cll}
			\displaystyle \min_{\bm x \in \cal X} & \displaystyle \beta_n \displaystyle \max_{\{\rv \zeta_j\}_{j \in [n]}} \displaystyle \frac{1}{n} \sum^n_{j = 1} h(\bm x, \rv \zeta_j) + (1-\beta_n) \displaystyle \sum^n_{i = 1} \frac{1}{n} \cdot h(\bm x, \rv \xi_i) \\
			\text{s.t.} & \displaystyle \frac{1}{n} \sum^n_{j = 1} d^p(\xi_j , \zeta_j) \le \epsilon^p_n,
		\end{array} 
  }
	\end{equation}
which is a finite-dimensional optimization. \stp
\end{corollary}

\subsection{Numerical Solution} 
The algorithm below, adapted from stochastic gradient descent (SGD)~\cite{kiefer1952stochastic}, provides a numerically iterative method to solve \eqref{eq:solution-BDR-phi} and \eqref{eq:solution-BDR-wasserstein} for gradient-based learning (e.g., neural networks).
\begin{algo}[BDR-GD to Solve \eqref{eq:solution-BDR-phi} and \eqref{eq:solution-BDR-wasserstein}]\label{algo:grad-descent}
With probability $\beta_n$ we use the gradient of the DRO term $\max_{\bm \mu} \sum^n_{i = 1} \mu_i \cdot h(\bm x, \rv \xi_i)$ or $\max_{\{\rv \zeta_j\}_{j \in [n]}} \frac{1}{n} \sum^n_{j = 1} h(\bm x, \rv \zeta_j)$, and with probability $1 - \beta_n$ we use the gradient of the SAA term $\frac{1}{n} \sum^n_{i = 1} h(\bm x, \rv \xi_i)$. 
For example, in the $t$-th iteration step, $\xi_t$ is sampled from $\Pnh$ and $p_t$ is sampled from the uniform distribution $\bb U_{(0,1]}$. Then the stochastic gradient, with respect to $\bm x$, 
\[
\bm g_{\bm x, t}=\begin{cases}
\nabla_{\bm x} h(\bm x, \xi_t),  &  \beta_n \leq p_t,\\
\nabla_{\bm x} \max_{\zeta_t} h(\bm x, \zeta_t) \text{ } \text{s.t.} \text{ }  d^p(\xi_t, \zeta_t)<\epsilon^p, &  \beta_n > p_t,
\end{cases}
\]
is calculated to update the hypothesis parameter $\bm x$.
\stp
\end{algo}

\subsection{Hyper-Parameter Tuning}
As demonstrated by the statistical properties in Theorem \ref{thm:gen-err}, the generalization performance of BDR learning is significantly influenced by the value of the hyper-parameter $\beta_n$. However, as highlighted in Remark \ref{rem:BDR-better-DRO}, the optimal value $\beta^*_n$ for $\beta_n$ cannot be theoretically determined due to its dependence on the unknown true distribution $\Po$. 
Therefore, in practice, $\beta_n$ can be empirically tuned using, e.g., grid search, cross-validation, and bootstrapping. This is a common practice of hyperparameter searching in, e.g., regularized SAA learning \eqref{eq:regularization-method} and DRO learning \eqref{eq:dro-method}. Experiments in Section \ref{sec:experiment} show that it is computationally lightweight to find some $\beta_n$ such that BDR can outperform both DRO and SAA.

\section{Practical Insights from BDR Learning}
Suppose that $\bm \mu^*$ solves \eqref{eq:solution-BDR-phi} and $\{\rv \zeta^*_j\}_{j \in [n]}$ solves \eqref{eq:solution-BDR-wasserstein}. Corollaries \ref{cor:solution-BDR-phi} and \ref{cor:solution-BDR-wasserstein} motivate two important insights in Examples \ref{rem:weight-modification} and \ref{rem:data-augmentation}, respectively.

 \begin{example}[Sample Weight Modification]\label{rem:weight-modification}
In ERM learning \eqref{eq:erm-method}, we work on equal-weighted $n$ samples $\{\xi_i\}_{i \in [n]}$, while in DRO learning $\min_{\bm x \in \cal X} \max_{\P \in B_{\epsilon_n, \phi}(\Pnh)} \E_{\P} h(\bm x, \rv \xi)$ with the $\phi$-divergence ball, the weights of samples $\{\xi_i\}_{i \in [n]}$ are modified into $\bm \mu^*$. However, in BDR learning \eqref{eq:solution-BDR-phi}, the weight of $\xi_i$ is given by
$\beta_n \mu^*_i + (1-\beta_n)/n$. 
\stp
\end{example}

An application of Example \ref{rem:weight-modification} is \quotemark{hard sample mining} \cite{shrivastava2016training}, where $\beta_n$ balances worst-case weight $\mu^*_i$ and homogeneous weight $1/n$ for sample $\xi_i$.

\begin{example}[Data Augmentation]\label{rem:data-augmentation} 
ERM learning~\eqref{eq:erm-method} works on equal-weighted $n$ nominal samples $\{\xi_i\}_{i \in [n]}$, while DRO learning $\min_{\bm x \in \cal X} \max_{\P \in B_{\epsilon_n, p}(\Pnh)} \E_{\P} h(\bm x, \rv \xi)$ with the Wasserstein ball constructs equal-weighted $n$ adversarial samples $\{\rv \zeta^*_j\}_{j \in [n]}$. In contrast, BDR learning \eqref{eq:solution-BDR-wasserstein} leverages $2n$ samples $\{\zeta^*_j\}_{j \in [n]} \cup \{\xi_i\}_{i \in [n]}$ with weight $\beta_n/n$ for adversarial samples $\{\zeta^*_j\}_{j \in [n]}$ and weight $(1-\beta_n)/n$ for nominal samples $\{\xi_i\}_{i \in [n]}$: it enables data augmentation by combining DRO-generated adversarial samples and the nominal samples in SAA.
\stp
\end{example}

In robust deep learning,  DRO-based adversarial training is widely used but infamous for its poor performance due to conservatism~\cite{raghunathan2019adversarial}.
BDR learning, however, can mitigate this issue by incorporating SAA learning, which is shown by experiments in Subsection \ref{subsec:expt-deep}.

\section{Applications and Experiments}\label{sec:experiment}
We show the practical benefits of the BDR learning framework through experimental results on real-world tasks such as 2D image and 3D point cloud classifications. Support vector machines and deep neural networks are specifically leveraged. 
All the source codes are available online at GitHub: \url{https://github.com/Spratm-Asleaf/Robustness-Specificity}.

\subsection{Linear Model: BDR Support Vector Machine}\label{sec:expt-svm}

We consider the binary classification problem on MNIST dataset \cite{lecun1998mnist} to distinguish similar handwritten digits 4 and 9. We adopt the support vector machine (SVM) as the classification algorithm and solve the problem under the frameworks of BDR, DRO, and SAA, respectively. Denote the $i$-th image's pixel vector as $\rvu I_i \in \R^{784}$ and its label as $Y_i \in \{-1,1\}$, i.e., $\xi_i \defeq (\rvu I_i, Y_i)$. We choose the order-1 Wasserstein distance to define a distributional ball under the metric \cite{shafieezadeh2019regularization}
\begin{equation}\label{eq:wasserstein-distance-metric}
d\left(\rv \xi_i, \rv \xi_j \right) \defeq \left\|\rvu I_i - \rvu I_j \right\|_{\infty} + \kappa \cdot \mathbbm{1}_{\left\{Y_i \neq Y_j\right\}}, 
\end{equation}
where $\left\| \cdot \right\|_{\infty}$ denotes the $\infty$-norm and  $\kappa$ quantiﬁes the cost of reversing a label. Hinge loss is used in SVM, i.e. 
\[
h(\bm x, \xi) = h(\bm x, (\rvu I, Y)) \defeq  \max\{1- Y \cdot \langle \bm x, \rvu I \rangle,~0\}.
\]
It can be derived from \eqref{eq:solution-BDR-wasserstein} and \cite[Cor.~15]{shafieezadeh2019regularization} that the BDR formulation is a linear program; see Appendix \ref{append:bdr-form-svm} for technical details. We conduct 100 independent trials, in each of which, 80\% of the images are randomly selected to train the model and the remaining 20\% images are used for testing. For BDR, we choose $\beta$ from $\{0.3, 0.5, 0.7\}$. For BDR and DRO, radius $\epsilon$ is chosen from $\{a \times 10^{-b} ~|~ a = 1,\cdots,9, b = 4,3,2\}$ and $\kappa$ is chosen from $\{0.1, 0.25, 0.5, 0.75\}$. The results are shown in Fig. \ref{fig:fig1}. 
\begin{figure}[!htbp]
	\centering
	\subfigure[Accuracy against $\epsilon$ \& $\kappa$ ]{
		\includegraphics[height=3.2cm]{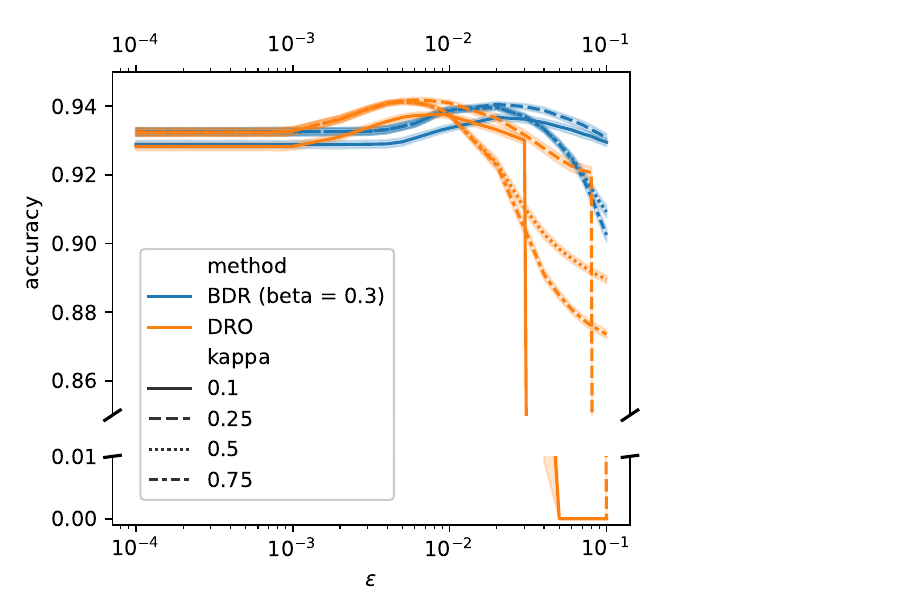}
		\label{fig:fig1b}
	}\hspace{-1.0em}
	\subfigure[Box plot of accuracy]{
		\includegraphics[height=3.2cm]{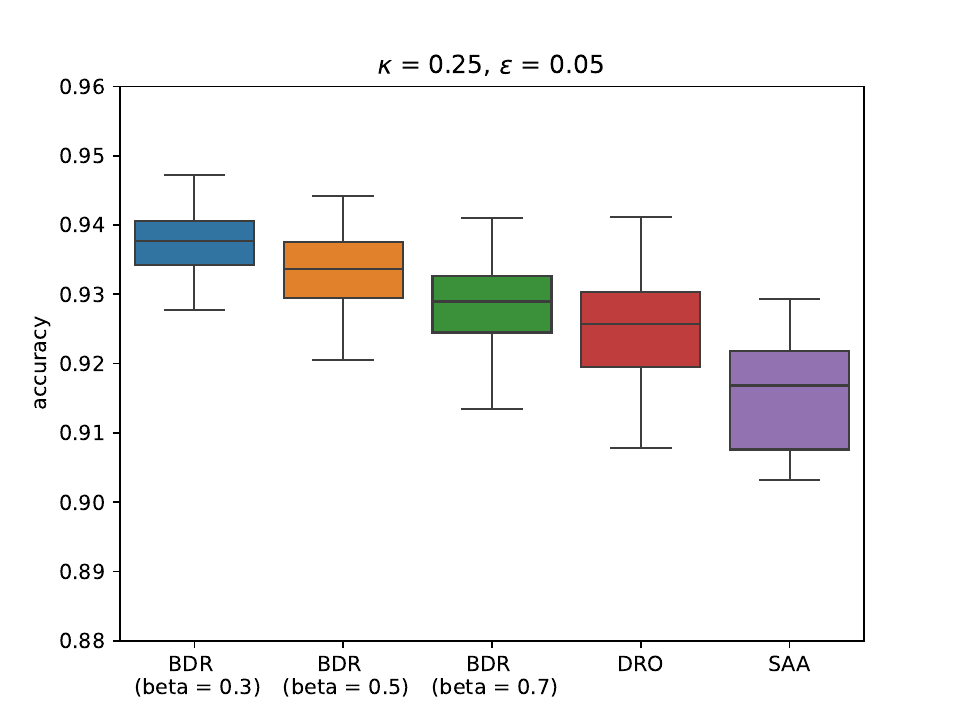}
		\label{fig:fig1c}
   \hspace{-1.5em}
	}
	\caption{Average test accuracy for 4 vs 9 over 100 trials. Averaged CPU times (seconds): BDR = 68, DRO = 66, and SAA = 7.}
	\label{fig:fig1}
\end{figure}

It can be seen in Fig. \ref{fig:fig1b} that the performances of BDR and DRO are significantly affected by $\epsilon$ and $\kappa$. The test accuracy first increases when $\epsilon$ increases but drops afterwards; the peak occurs in the range of $\epsilon \in [0.005, 0.05]$. This phenomenon agrees with our claim that the radii of the ambiguity sets can neither be too large nor too small: If the ambiguity sets are too small, robust methods cannot provide sufficient robustness; however, if the ambiguity sets are too large, robust methods are too conservative. 
Among different $\kappa$, $\kappa = 0.25$ works best for both BDR and DRO. 
Fig. \ref{fig:fig1c} shows an accuracy comparison among BDR (with different $\beta$), DRO, and SAA, under $\kappa = 0.25$ and $\epsilon = 0.05$ as selected above, where BDR with $\beta = 0.3$ has higher accuracy compared to that with $\beta = 0.5$ and $\beta = 0.7$.  Fig. \ref{fig:fig1b} also supports our claim that BDR is less conservative than DRO---To be specific, DRO is sensitive to the choice of $\epsilon$ because a slight change of $\epsilon$ can lead to a large change in accuracy (especially around $\epsilon = 0.07$); in contrast, BDR is more robust to the choice of $\epsilon$. 

For more experimental results of BDR SVM on MNIST and UCI data sets, as well as running times, see Appendix \ref{append:experiments-svm}.

\subsection{Deep Learning Model: BDR Learning}\label{subsec:expt-deep}
We present an implementation of deep BDR learning (DBDRL) and demonstrate the potential of our BDR model in enhancing the performance of deep models on various tasks.

\textbf{Tasks}: We apply the proposed  BDR model to 2D image classification tasks using MNIST~\cite{lecun1998mnist}, CIFAR-10, and CIFAR-100~\cite{krizhevsky2009learning} datasets, as well as 3D point cloud classification utilizing ModelNet40~\cite{wu20153d} dataset. To evaluate the generalization capacity of our method, we perform experiments under a low-shot data setting; that is, the model is learned on a subset of the training dataset. This setup means that a learning model yielding higher testing performance on a small training dataset has a better generalization capability. 

\textbf{Implementation}: We consider the objective of DBDRL as presented in~\eqref{eq:solution-BDR-wasserstein}. Specifically, we employ the convex cross-entropy loss~\cite{wang2020comprehensive} as the function $h$ for our learning.
Additionally, we implement the BDR-GD in Algorithm~\ref{algo:grad-descent} for DBDRL. 
The DRO term in BDR-GD is actualized through Adversarial Training (AT) techniques, with the employment of a specific Projected Gradient Descent (PGD) method~\cite{madry2018towards} to perform the maximization and construct adversarial samples. For PGD implementation, we use the order-2 distance for the constraints; using notations in \eqref{eq:wasserstein-distance-metric}, an example is given by
\[%\begin{equation}
d\left(\rv \xi_i, \rv \xi_j \right) = \left\|\rvu I_i - \rvu I_j \right\|_{2}.
\]%\end{equation}
We follow the official implementation to train our models in both 2D and 3D tasks except for the low-shot data setting and BDR-GD utilization. Further details, such as the parameters of training and PGD, are put in Appendix \ref{append:experiments-resnet}. 

\textbf{Results of MINIST}: We implement the WideResNet-28 (WRN)~\cite{zagoruyko2016wide} for our 2D experiments. We first demonstrate the capability of DBDRL with varying $\beta$ values on the MNIST dataset. As depicted in Fig.~\ref{fig:error-rate-mnist}, the best $\beta^*$, which is an estimation of $\beta^*_n$ in Theorem~\ref{thm:gen-err}, diminishes as the volume of training data escalates, corroborating the property of \eqref{eq:dro-saa-obj-general}. Moreover, we observe that the best BDR models consistently outperform both their DRO and SAA counterparts; the advantage of BDR is especially obvious with smaller training data set. This is consistent with the theoretical analyses in Section~\ref{subsec:properties-BDR}.
\begin{figure}[!htbp]
\centering
\includegraphics[width=0.49\textwidth]{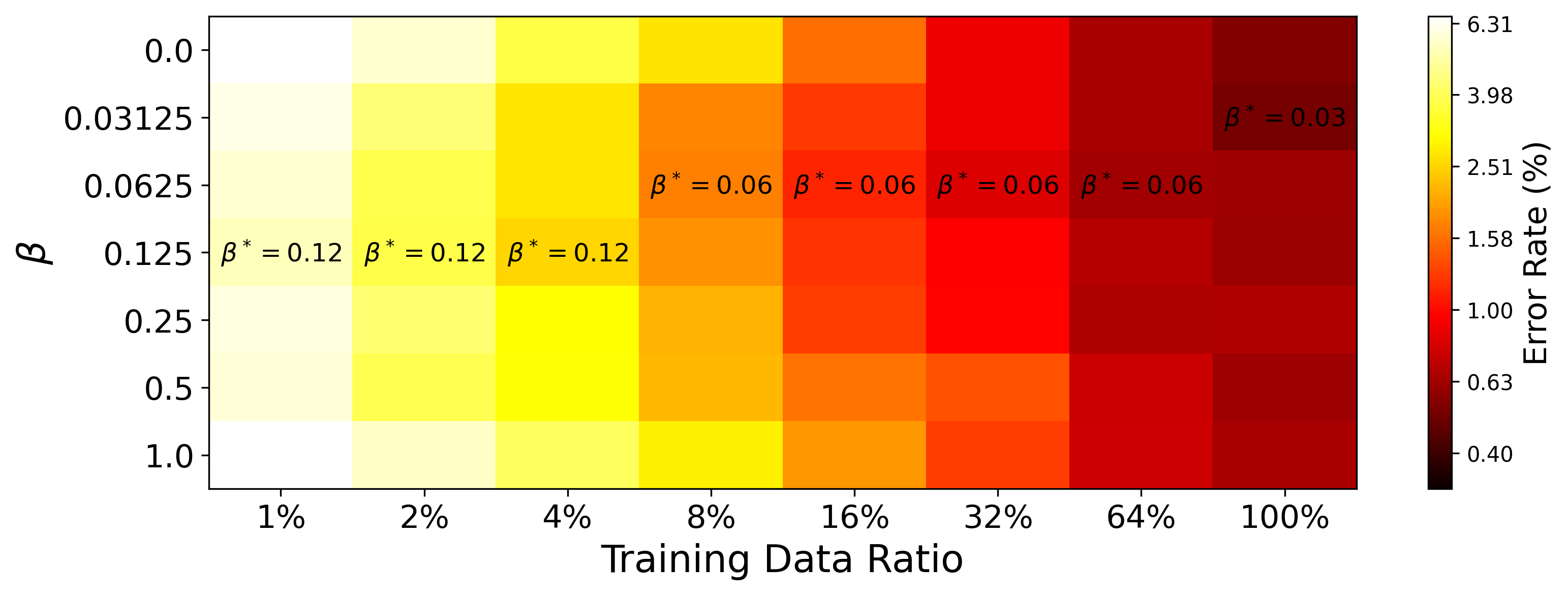}
\caption{Error rate of models trained by partial training sets on MNIST test set. Various $\beta$ values are used during training: $\beta=0$ for SAA learning, $\beta=1$ for DRO learning, and $\beta^*$ indicating the best value among various $\beta$ for BDR learning.
 }
 \label{fig:error-rate-mnist}
\end{figure}

\textbf{Parameter Tuning}: To obtain $\beta^*$ during training, we adopt a validation-based search strategy: we leverage a subset of the training dataset (20\% in our setting) as a validation set to search for a decent $\beta$. We highlight that the search cost is not significantly high, as it is found that a low precision of estimation can still enhance performance in practice. In later experiments, we restrict our search of $\beta^*$ to a smaller set, \textit{i.e.}, $\{0.5, 0.1, 0.05, 0.01\}$, and employ early stopping techniques to expedite the search process.

\textbf{Main Results of 2D and 3D Classification}: With the same tuning strategy for $\beta_n$, we showcase the superiority of our methods on CIFAR datasets in Table~\ref{tab:acc-cifar}. The used model is  WRN-28 which is the same as MNIST experiments. We also employ DBDRL in 3D point cloud classification by implementing two models, PointNet~\cite{qi2017pointnet} and DGCNN~\cite{wang2019dynamic}. We utilize the above search method of $\beta^*$ and demonstrate the consistent best performances of our BDR methods in Table~\ref{tab:acc-modelnet40}. Notably, DBDRL can improve the model performance from both the SAA learning and DRO learning across all tasks. Additionally, we note that $\beta^*$ estimation may not be accurate, as our search is limited to only a small set $\{0.01, 0.05, 0.1, 0.5\}$. However, high estimation accuracy of $\beta^*$ is not critical in practice because, as depicted in Fig. \ref{fig:error-beta}, a wide range of $\beta$ values can make BDR outperform the DRO and SAA. Overall, it is computationally lightweight to search for decent $\beta$s that enable BDR to outperform DRO and SAA. To illustrate this, we provide a detailed complexity analysis in  Appendix \ref{append:complexity} to show that the above search process can be done with trivial effort while achieving better performance.

\begin{table}[!htbp]
\centering \small
\caption{Accuracy (\%) of image classification on CIFAR-10 \& CIFAR-100 under low-shot data (10\% or 50\% training data) setting. }
\begin{tabular}{c|cc|cc}
\toprule
\multirow{2}{*}{Method} & \multicolumn{2}{c|}{CIFAR-10} & \multicolumn{2}{c}{CIFAR-100} \\  
                               & 10\%      & 50\%     & 10\%        & 50\%        \\ \midrule\midrule
DRO                            & 64.9       & 86.3       & 26.2     & 61.3     \\
SAA                            & 63.5               & 87.0                 & 24.1             & 61.6             \\ \midrule
BDR                            & \textbf{66.5}                & \textbf{87.3}               & \textbf{26.9}               &  \textbf{63.4}            \\
$(\beta^*)$  & (0.05)                & (0.05)                & (0.1)               & (0.05)             \\

\bottomrule
\end{tabular}
\label{tab:acc-cifar}
\end{table}

\begin{table}[!htbp]
\centering
\caption{Accuracy (\%) of point cloud classification on ModelNet40 by different learning methods. Different training data ratios are utilized. The estimated $\beta^*$ for each BDR learning is also given. }
\begin{tabular}{lccccc}
\toprule
\multirow{2}{*}{Model}    & \multirow{2}{*}{Data ratio} & \multicolumn{3}{c}{Method}     & \multirow{2}{*}{$\beta^*$} \\ \cline{3-5}
                          &                             & DRO   & SAA   & BDR            &                       \\ \midrule\midrule
\multirow{3}{*}{PointNet} & 5\%                         & 72.9 & 72.3 & \textbf{72.9} & 0.5                   \\
                          & 10\%                        & 79.6 & 79.4 & \textbf{80.6} & 0.1                   \\
                          & 100\%                       & 88.7 & 89.1  & \textbf{89.8} & 0.05                  \\ \midrule
\multirow{3}{*}{DGCNN}    & 5\%                         & 78.4 & 77.1 & \textbf{79.86} & 0.1                   \\
                          & 10\%                        & 85.1 & 84.3 & \textbf{85.8} & 0.1                   \\
                          & 100\%                       & 91.9 & 92.1 & \textbf{92.8} & 0.01                  \\ \bottomrule
\end{tabular}
\label{tab:acc-modelnet40}
\end{table}

\begin{figure}[!htbp]
\centering
\includegraphics[width=0.4\textwidth]{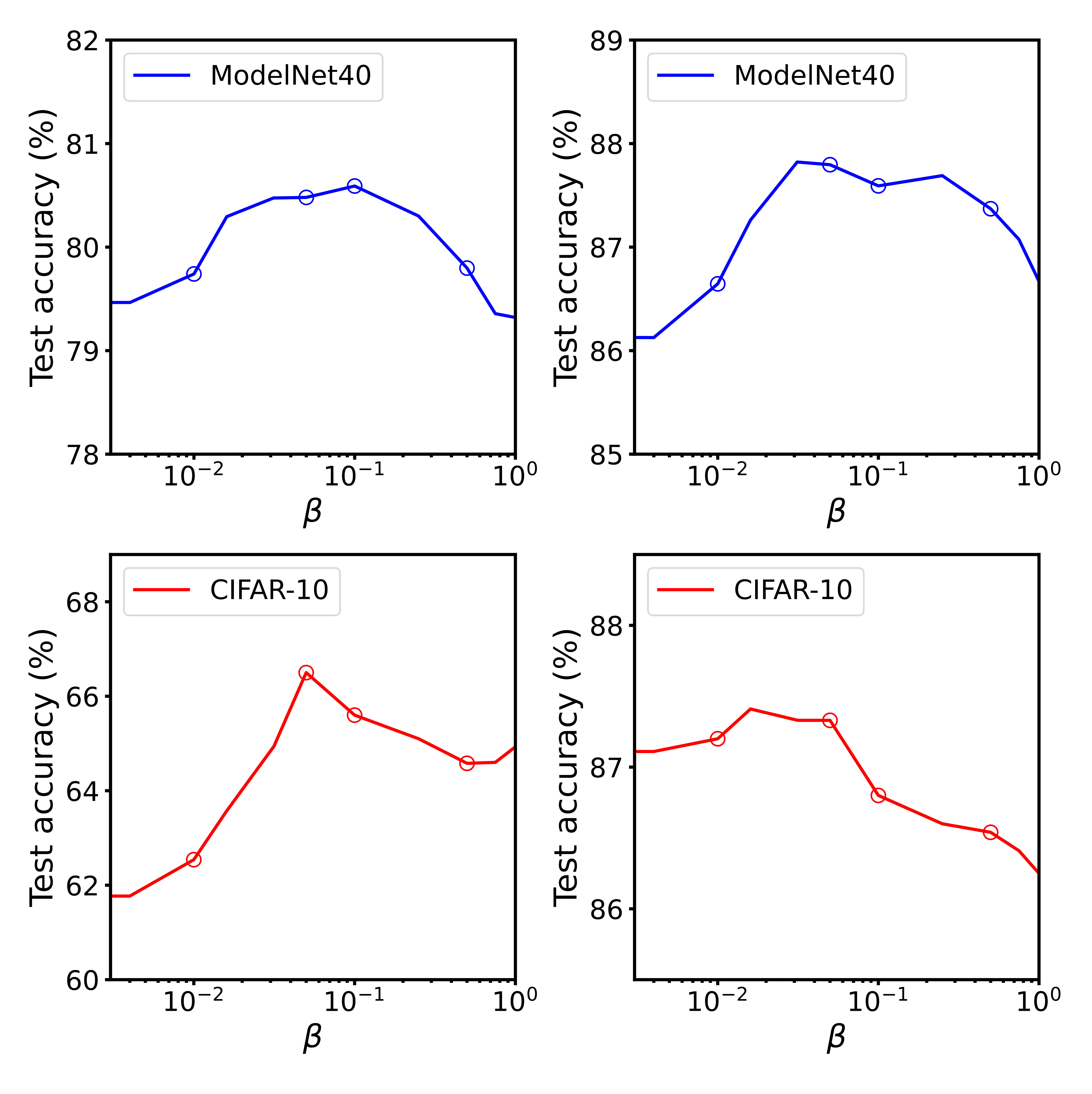}
\caption{Test set accuracy \textit{v.s.} $\beta$ across various tasks. Upper Panel: PointNet on ModelNet40 with 10\% (left) and  50\% (right) training data. Lower Panel: WRN-18 on CIFAR-10 with 10\% (left) and  50\% (right) training data. The marker ``$\circ$'' stands for searching set of $\beta$: i.e., $\{0.01, 0.05, 0.1, 0.5\}$. (NB: $\beta=0$ for SAA learning, $\beta=1$ for DRO learning.)}
\label{fig:error-beta}
\end{figure}

\section{Conclusions}\label{sec:conclusion}
This paper proposes the Bayesian distributionally robust learning framework \eqref{eq:BDR-opt} or \eqref{eq:bdr-method} that generalizes the Bayesian method, distributionally robust optimization method, and regularization method; see Remark \ref{rem:BDR-interpret}. 
The new framework reveals that there exists a trade-off between the robustness to the distributional uncertainty and the specificity to the empirical information; see Remark \ref{rem:trade-off}. The new framework also suggests the design methods of the prior distribution $\bb Q$ in the Bayesian method \eqref{eq:bayesian-method} and the regularizer $f(\bm x)$ in the regularization method \eqref{eq:regularization-method} (see Remark \ref{rem:BDR-interpret}), and shows that BDR learning can be less conservative than DRO learning (see Theorem \ref{thm:gen-err}, Remark \ref{rem:BDR-better-DRO}, Examples \ref{ex:gen-err} and \ref{ex:gen-err-lin-reg}, and Figs. \ref{fig:fig1}, \ref{fig:error-rate-mnist}, and \ref{fig:error-beta}).
The asymptotic (i.e., consistencies and asymptotic normalities in Theorem \ref{thm:asym-properties}) and non-asymptotic (i.e., generalization bounds in Theorem \ref{thm:gen-err} and unbiasedness in Theorem \ref{thm:unbiasedness}) properties, and the solution method (i.e., Corollaries \ref{cor:solution-BDR-phi} and \ref{cor:solution-BDR-wasserstein}) of the new framework are studied. In addition, the BDR learning framework reveals important insights from the perspective of data augmentation; see Examples \ref{rem:weight-modification} and \ref{rem:data-augmentation}. Experiments on diverse real-world datasets demonstrate the practical usefulness of the proposed BDR model. 

The future research direction is to study alternatives for the Dirichlet-process priors for the second-order probability measure $\bb Q$ in the Bayesian model \eqref{eq:bayesian-method}, which possibly motivates other new robust learning models than the proposed BDR models in \eqref{eq:BDR-opt} and \eqref{eq:bdr-method}. Possible replacements are Dirichlet-process mixture priors \cite[Chap.~5]{ghosal2017fundamentals}, tail-free process priors \cite[Sec.~3.6]{ghosal2017fundamentals}, among many others.

\appendices

\section{Appendices of Section \ref{sec:prelimilary}}\label{append:prelimilary}

\subsection{Notations}\label{append:notations}
Notations used in this paper are summarized in Table \ref{tab:notation}.

\begin{table}[!htbp]
	\centering
	\caption{Full notation List}
	\label{tab:notation}
	\begin{tabular}{ll}
		\toprule
		\textbf{Symbol} & \textbf{Interpretation} \\
		\midrule
		$\cal M(\Xi)$
		&
		\tabincell{l}{
			all distributions on $(\Xi, \cal B_{\Xi})$ where $\cal B_{\Xi}$ is the \\ Borel $\sigma$-algebra on $\Xi$}
		\\
		\midrule
		$\cal B_{\cal M(\Xi)}$
		&
		\tabincell{l}{
			Borel $\sigma$-algebra on $\cal M(\Xi)$
		} 
		\\
		\midrule
		$\Po$
		&
		\tabincell{l}{
			true population distribution
		}
		\\
		\midrule
		$\Pnh$
		&
		\tabincell{l}{
			empirical distribution supported on $n$ \\ i.i.d. samples
		}
		\\
		\midrule
		$\Ph$
		&
		\tabincell{l}{
			a prior estimate of $\Po$ based on prior knowledge
		}
		\\
		\midrule
		$\Pb$
		&
		\tabincell{l}{
			reference distribution working as a proper \\ estimate of $\Po$, which can be the empirical $\Pnh$ \\ or the prior $\Ph$, among many others
		}
		\\
		\midrule
		$\Po^n$
		&
		\tabincell{l}{
			$n$-fold product measure induced by $\Po$ \\ (i.e., joint distribution of $n$ i.i.d. samples)
		}
		\\
		\midrule
		$[n]$
		&
		\tabincell{l}{
			$[n] \defeq \{1,2,\ldots,n\}$, the running index set
		}
		\\
		\midrule
		$\Delta(\P, \Pnh)$
		&
		\tabincell{l}{
			statistical similarity measure between $\P$ and $\Pnh$; \\ $\Delta$ can be any possible divergences or statistical \\ distances
		}
		\\
		\midrule
		$B_{\epsilon}(\Pnh)$
		&
		\tabincell{l}{
			$\defeq \{\P \in \cal M(\Xi)| \Delta(\P, \Pnh) \le \epsilon\}$, closed \\ distributional ball with radius $\epsilon$ and center $\Pnh$
		}
		\\
		\midrule
		$N(\bm \mu, \bm \Sigma)$
		&
		\tabincell{l}{
			Gaussian distribution with mean $\bm \mu$ and \\ covariance $\bm \Sigma$
		}
		\\
		\midrule
		$\ascvg$
		&
		\tabincell{l}{
			converges almost surely
		}
		\\
		\midrule
		$\pcvg$
		&
		\tabincell{l}{
			converges in probability
		}
		\\
		\midrule
		$\dcvg$
		&
		\tabincell{l}{
			converges in distribution
		}
		\\
		\midrule
		$o_p(1)$
		&
		\tabincell{l}{
			if a sequence $a_n = o_p(1)$, then $a_n$ converges \\ to zero in probability
		}
		\\
		\midrule
		$d(\bm x, \bm y)$
		&
		\tabincell{l}{
			distance between two points $\bm x$ and $\bm y$
		}
		\\
		\midrule
		$d(\bm x, \cal X)$
		&
		\tabincell{l}{
			$\defeq \inf_{\bm y \in \cal X} \|\bm x - \bm y\|$, \\ distance between the point $\bm x$ and the set $\cal X$
		}
		\\
		\midrule
		$d(\cal X, \cal Y)$
		&
		\tabincell{l}{
			$\defeq \sup_{\bm x \in \cal X} d(\bm x, \cal Y)$, \\
            distance between the two sets $\cal X$ and $\cal Y$
		}
		\\
		\midrule
		$\E_{\P}[\cdot], \D_{\P}[\cdot]$
		&
		\tabincell{l}{
			the expectation operator and the covariance \\ operator, respectively, with respect to the \\ distribution $\P$
		}
		\\
		\midrule
		$\grad_{\bm x} h(\bm x_0, \rv \xi)$
		&
		\tabincell{l}{
			the gradient, i.e., Jacobian, of $h(\bm x, \rv \xi)$ with \\ respect to $\bm x$ evaluated at $\bm x_0$
		}
		\\
		\midrule
		$\grad^2_{\bm x} h(\bm x_0, \rv \xi)$
		&
		\tabincell{l}{
			the second-order gradient, i.e., Hessian, of \\ $h(\bm x, \rv \xi)$ with respect to $\bm x$ evaluated at $\bm x_0$
		}
		\\
		\midrule
		$\bm V^{-\top} \defeq [\bm V^{-1}]^{\top}$
		&
		\tabincell{l}{
			the transpose of the inverse of the matrix $\bm V$
		}
		\\
            \midrule
		$[\bm a,\bm b]$ and $[\bm a;\bm b]$
            & \tabincell{l}{ MATLAB notation for row and column \\ concatenation of $\bm a$ and $\bm b$, respectively}
            \\
		\bottomrule
	\end{tabular}
\end{table}

\subsection{Similarity Measures of Distributions and Distributional Balls}\label{subsec:distributional-balls}

\subsubsection{\captext{$\phi$}-Divergence}\label{subsec:phi-divergence}
Suppose $\P$ is absolutely continuous with respect to $\Pb$. 
Let $\phi:\R_{+} \to \{\R \cup +\infty\}$ denote a convex function that satisfies $\phi(1) = 0$ and $0\phi(0/0) = 0$.
The $\phi$-divergence (i.e., $f$-divergence) of $\P$ from $\Pb$, generated by $\phi$, is defined as
\begin{equation}\label{eq:phi-divergence}
	F_\phi(\P\|\Pb) = \int_{\Xi} \phi\left(\frac{\d \P}{\d \Pb}\right) \Pb(\d \rv \xi),
\end{equation}
where ${\d \P}/{\d \Pb}$ is the Radon--Nikodym derivative of $\P$ with respect to $\Pb$. When $\phi(t) \defeq t\ln t$ for all $t > 0$, the $\phi$-divergence specifies the well-known Kullback--Leibler divergence; cf. \cite[Table 2]{ben2013robust}.

A $\phi$-divergence distributional ball with radius $\epsilon \ge 0$ and center $\Pb$ is defined as
\[
B_{\epsilon, \phi} (\Pb) \defeq \{\P \in \cal M(\Xi) | F_\phi(\P \| \Pb) \le \epsilon\}.
\]
If $\epsilon = 0$, the ball $B_{\epsilon,\phi}(\Pb)$ reduces to the singleton that contains only $\Pb$. In some literature, the ball is also defined as $B_{\epsilon, \phi} (\Pb) \defeq \{\P \in \cal M(\Xi) | F_\phi(\Pb \| \P) \le \epsilon\}$ where $\P$ and $\Pb$ are swapped. The two versions are not equivalent because the $\phi$-divergence is not guaranteed to be symmetric in general.

\subsubsection{Wasserstein Distance}\label{subsec:wasserstein-distance}
The order-$p$ Wasserstein distance between the two distributions $\P$ and $\Pb$ is defined as
\begin{equation}\label{eq:wasserstein-distance}
\begin{array}{cl}
	W_p(\P, \Pb) &= {\left[\inf_{\pi \in \cal M (\Xi \times \Xi)} \E_{\pi} d^p(\rv \xi_1, \rv \xi_2)\right]}^{\frac{1}{p}} \\
    &= {\left[\inf_{\pi \in \cal M (\Xi \times \Xi)} \int_{\Xi \times \Xi} d^p(\rv \xi_1, \rv \xi_2) \pi(\d \rv \xi_1, \d \rv \xi_2) \right]}^{\frac{1}{p}},
\end{array}
\end{equation}
where $d$ is a distance on $\Xi$, $p \ge 1$, and $\pi$ is a joint distribution on $\Xi \times \Xi$ with marginals $\P$ and $\Pb$.

An order-$p$ Wasserstein distributional ball with radius $\epsilon$ and center $\Pb$ is defined as
\[
B_{\epsilon,p} (\Pb) \defeq \{\P \in \cal M(\Xi) | W_{p}(\P, \Pb) \le \epsilon\}.
\]
If $\epsilon = 0$, the ball $ B_{\epsilon,p}(\Pb)$ reduces to the singleton that contains only $\Pb$.

Wasserstein balls admit the following concentration properties. Suppose the true population distribution $\Po$ has a light tail: That is, there exist $\alpha > p \ge 1$ (but $p \ne k/2$) and finite $A>0$ such that $\E_{\Po}\left[\exp \left(\|\rv \xi\|^{\alpha}\right)\right] \leq A$ (recall that $k$ is the dimension of $\rv \xi$). Then, there exist constants $c_{1}, c_{2}>0$ such that 
\begin{equation}\label{eq:wasserstein-concentration}
	\Po^n \left[\Po \in {B}_{\epsilon_n, p}(\Pnh)\right] \geq 1-\eta
\end{equation}
holds, for any $\eta \in(0,1]$, when
\begin{equation}\label{eq:wasserstein-concentration-epsilon}
\epsilon_n \ge 
\begin{cases}
	\left(\frac{\log \left(c_{1} / \eta\right)}{c_{2} n}\right)^{\min \{1 / k, 1 / 2\}} & \text { if } n \geq \frac{\log \left(c_{1} / \eta\right)}{c_{2}}, \\
	\left(\frac{\log \left(c_{1} / \eta\right)}{c_{2} n}\right)^{1 / \alpha} & \text { if } n<\frac{\log \left(c_{1} / \eta\right)}{c_{2}}.
\end{cases}
\end{equation}
Note that $c_{1}$ and $c_{2}$ are determined by $\alpha$, $A$, and $k$. This result is attributed to \cite[Thm. 18]{kuhn2019wasserstein}. The difficulty of applying this result in practice is that the involved constants $\alpha$ and $A$ cannot be exactly obtained because the population distribution $\Po$ is unknown, and so are $c_1$ and $c_2$.

When the support set $\Xi$ is finite and bounded (i.e., $\Po$ is discrete), there exist concentration properties of $\Pnh$ with respect to the Wasserstein distance that do not depend on unknown constants; see, e.g., \cite[pp. 42]{chen2020distributionally}.

\subsection{Wasserstein DRO Models}\label{subsec:wasserstein-reformulation}
\subsubsection{Existence of The Solution of Wasserstein DRO Models}
Suppose $(\Xi, d)$ is a proper,\footnote{A metric space $(\Xi, d)$ is proper if for any $\epsilon > 0$ and $\rv \xi_0 \in \Xi$, the closed $\epsilon$-ball $B_{\epsilon}(\rv \xi_0) \defeq \{\rv \xi \in \Xi | d(\rv \xi, \rv \xi_0) \le \epsilon \}$, is compact.} complete, and separable metric space, $h(\bm x, \rv \xi)$ is upper semi-continuous in $\rv \xi$ on $\Xi$ and $\E_{\Pb} |h(\bm x, \rv \xi)| < \infty$ for every $\bm x$, and $\Pb$ has a finite $p$-th moment: That is, for every $\rv \xi_0 \in \Xi$, we have
$
\int_{\Xi} d^p(\rv \xi, \rv \xi_0) \Pb(\d \rv \xi) < \infty.
$ 
Then, for every $\bm x$, the optimal value of the Wasserstein DRO problem
\begin{equation}\label{eq:wasserstein-DRO}
 \max_{\P: W_p(\P, \Pb) \le \epsilon}  \int_{\Xi} h(\bm x, \rv \xi) \P(\d \xi)
\end{equation} 
is finite if and only if there exist $\rv \xi_0 \in \Xi$ and $c_1(\bm x) > 0$ such that 
\begin{equation}\label{eq:wasserstein-finite}
	h(\bm x, \rv \xi) \le c_1(\bm x) [1 + d^p(\rv \xi, \rv \xi_0)], ~~~ \forall \rv \xi \in \Xi.
\end{equation}
In addition, the optimal value is attainable (by one $\P^*$ such that $W_p(\P^*, \Pb) \le \epsilon$) if there exist
$\rv \xi_0 \in \Xi$, $c_1(\bm x) > 0$, and $c_2 \in (0, p)$ such that 
\begin{equation}\label{eq:wasserstein-attainable}
	h(\bm x, \rv \xi) \le c_1(\bm x) [1 + d^{c_2}(\rv \xi, \rv \xi_0)], ~~~ \forall \rv \xi \in \Xi.
\end{equation}
The results above can be seen in, e.g., \cite{yue2021linear,chen2020distributionally}. Note that \eqref{eq:wasserstein-finite} is in analogy to the Lipschitz continuity which limits the \quotemark{change rate} of a function. To clarify further, for example, by letting $p \defeq 1$ and $d \defeq \|\cdot\|$ (i.e., the metric $d$ is induced by a norm $\|\cdot\|$), we can see that \eqref{eq:wasserstein-finite} is in analogy to $h(\bm x, \rv \xi) \le h(\bm x, \rv \xi_0) + L(\bm x) \|\rv \xi - \rv \xi_0\|$, for every $\rv \xi,\rv \xi_0 \in \Xi$, where $L(\bm x) > 0$ is the Lipschitz constant. For this reason, in literature, e.g., \cite{chen2020distributionally,gao2022distributionally}, \eqref{eq:wasserstein-finite} is called the \quotemark{finite-growth-rate} condition for the function $h$.

In this paper, for practicality, we consistently assume that the condition \eqref{eq:wasserstein-attainable} is satisfied so that it is safe to replace the supremum with the maximum in the DRO model.

\subsubsection{Reformulation of Wasserstein DRO Models}
According to, e.g., \cite[Thm. 1]{blanchet2021statistical} and \cite[Thm. 1]{gao2022distributionally},\footnote{The finite growth-rate assumption for the function $h$ in \cite{gao2022distributionally} is equivalent to require \eqref{eq:wasserstein-finite}; see Lemma 2 therein.} the Wasserstein DRO problem \eqref{eq:wasserstein-DRO} is equivalent to its Lagrangian dual:\footnote{$\lambda$ is the dual variable for the constraint in \eqref{eq:wasserstein-DRO}.}
\begin{equation}\label{eq:wasserstein-DRO-dual}
\displaystyle \min_{\lambda \ge 0} \left\{ \lambda \epsilon^p +  \int_{\Xi} \max_{\rv \xi \in \Xi} \Big\{h(\bm x, \rv \xi) - \lambda \cdot d^p(\rv \xi, \rvb \xi) \Big\} \Pb(\d \rvb \xi) \right\}.
\end{equation}
If $\Pb = \sum^n_{i = 1} \mub_i \delta_{\rv \xi_i}$ is a discrete distribution, e.g., an empirical distribution, supported on $n$ points $\{\rv \xi_i\}_{i \in [n]}$, then \eqref{eq:wasserstein-DRO-dual} becomes
\begin{equation}\label{eq:DRO-dual-discrete}
	\displaystyle \min_{\lambda \ge 0} \left\{ \lambda \epsilon^p +  \sum^n_{i=1} \mub_i \max_{\rv \xi \in \Xi} \Big\{h(\bm x, \rv \xi) - \lambda \cdot d^p(\rv \xi, \rv \xi_i) \Big\} \right\}.
\end{equation}

\subsubsection{Support Set of Worst-Case Distributions}
If $\Pb$ is supported on $n$ points in $\Xi$, then the worst-case distribution solving \eqref{eq:wasserstein-DRO} is supported on at most $n+1$ points in $\Xi$; see \cite[Thm. 4]{yue2021linear}, \cite[Cor.~2]{gao2022distributionally}.

Special cases when $h$ is concave or piece-wise linear in $\rv \xi$ or when $\Pb \defeq \Pnh$ are discussed in, e.g., \cite{esfahani2018data,shafieezadeh2019regularization,kuhn2019wasserstein,zhao2018data}.

\subsection{Glivenko--Cantelli Class, Donsker Class, and Brownian Bridge}\label{append:statistical-concepts}
Consider a function class $\cal F \defeq \{f: \Xi \to \R\}$.
\begin{definition}[Glivenko--Cantelli Class]\label{def:GC}
	Suppose for every $f \in \cal F$, $\E_{\Po} f(\rv \xi)$ is defined\footnote{At least one of the positive part and the negative part of $f$ has finite integral.} and finite; that is, $f$ is $\Po$-integrable. The function class $\cal F$ is called $\Po$-Glivenko--Cantelli if 
	\begin{equation}\label{eq:ULLN}
		\sup_{f \in \cal F} |\E_{\Pnh} f(\rv \xi) - \E_{\Po} f(\rv \xi)| \ascvg 0.
	\end{equation}
	Intuitively, if $\cal F$ is a Glivenko--Cantelli class, then the uniform strong law of large numbers holds on $\cal F$.
	\stp
\end{definition}

\begin{definition}[Donsker Class]\label{def:Donsker}
	Consider an empirical process 
	\begin{equation}\label{eq:ep-at-f}
		\bb G_n(f) \defeq \sqrt{n} [\E_{\Pnh} f(\rv \xi) - \E_{\Po} f(\rv \xi)],~~~\forall f \in \cal F
	\end{equation}
	indexed by the function class $\cal F$. That is, $\{\bb G_n(f)| f \in \cal F\}$ in \eqref{eq:ep-at-f} is a stochastic process indexed by $\cal F$; the randomness comes from the (random) empirical measure $\Pnh$. Suppose for every $f \in \cal F$, $\D_{\Po} f(\rv \xi)$ is defined and finite; that is, $f$ is $\Po$-square-integrable. The function class $\cal F$ is called $\Po$-Donsker if the empirical (stochastic) process $\bb G_n$ converges in distribution to a Brownian bridge (stochastic) process:
	\begin{equation}\label{eq:UCLT}
		\bb G_n \dcvg \bb G_{\Po},
	\end{equation}
	where $\bb G_{\Po}$ is a zero-mean $\Po$-Brownian bridge on $\cal F$ with uniformly continuous sample paths with respect to the semi-metric $\sqrt{\D_{\Po}[f_1(\rv \xi) - f_2(\rv \xi)]}$ between $f_1 \in \cal F$ and $f_2 \in \cal F$; in addition, $\bb G_{\Po}(f)$ is tight for every $f \in \cal F$; i.e., $\sup_{f \in \cal F} |\bb G_{\Po}(f)| < \infty$ in $\Po$-probability.
	Intuitively, if $\cal F$ is a Donsker class, then the uniform central limit theorem\footnote{The uniform central limit theorem is also known as the functional central limit theorem as a random function(al) sequence (i.e., the empirical process) converges to a random function(al) (i.e., a Brownian bridge).} holds on $\cal F$.
	\stp
\end{definition}

A zero-mean $\Po$-\bfit{Brownian bridge} $\bb G_{\Po}$ on $\cal F$ is a Gaussian process on $\cal F$ satisfying the following two conditions:
\begin{enumerate}
	\item For every $f \in \cal F$, $\bb G_{\Po}(f)$ is a random variable with mean of zero and variance of $\D_{\Po}(f)$.
	\item For every integer $r$ and every possible collection of functions $\{f_1, f_2, \ldots, f_r\}$ taken from $\cal F$, the random vector 
	$%\[
	[\bb G_{\Po}(f_1),\bb G_{\Po}(f_2),\ldots,\bb G_{\Po}(f_r)]^\top
	$%\] 
    follows a $r$-dimensional multivariate Gaussian distribution with covariance between $\bb G_{\Po}(f_{i})$ and $\bb G_{\Po}(f_{j})$ being defined as
	$%\[
	\E_{\Po}f_{i} \cdot f_{j} - \E_{\Po}f_{i} \cdot \E_{\Po}f_{j},
	$ %\]
	for every $i,j \in [r]$.
\end{enumerate}
Since the values of the Gaussian process $\bb G_{\Po}$ at some functions $f \in \cal F$ are strictly zeros, without any randomness, the Gaussian process $\bb G_{\Po}$ is called a Brownian bridge because some values are tied, for example, when $f$ is $\Po$-almost everywhere constant.

\section{Examples Satisfying The Conditions in Theorem \ref{thm:asym-properties}}\label{append:condition-examples-in-theorem-asym}
The conditions C1)-C6) in Theorem \ref{thm:asym-properties} are not practically restrictive as they are standard for the DRO model \eqref{eq:dro-method} \cite{kuhn2019wasserstein,esfahani2018data,yue2021linear,gao2022finite} and the SAA model \eqref{eq:erm-method} \cite{van1996weak}, \cite[Chap.~19]{vdv1998asymptotic}, \cite[Chap.~5]{shapiro2009lectures}. The only new requirement is Condition C2); i.e., $\sqrt{n} \beta_n \to 0$, which is also mild. Some specific situations where the conditions C1)-C6) in Theorem \ref{thm:asym-properties} hold are given below.

Condition C1) holds if, for example, \eqref{eq:wasserstein-attainable} is satisfied;
 
Condition C2) holds if, for example, $\beta_n \defeq \frac{\alpha}{n + \alpha}$, for every $n$, where $\alpha \ge 0$ is a constant;\footnote{Recall from \eqref{eq:dro-saa-obj} that this rule is used in the Dirichlet process prior for a Bayesian non-parametric model.}

Condition C3) holds if, for example, one of the following is satisfied:
	\begin{enumerate}[a)]
		\item The function class $\cal H$ is finite and every element of $\cal H$ is $\Po$-integrable;
		\item The parameter space $\cal X$ is bounded, every element of $\cal H$ is $\Po$-integrable, and there exists a $\Po$-integrable function $m(\rv \xi)$ such that
		\begin{equation}\label{eq:H-Lip}
			|h(\bm x_1, \rv \xi) - h(\bm x_2, \rv \xi)| \le m(\rv \xi) \|\bm x_1 - \bm x_2\|,~~~\forall \bm x_1, \bm x_2 \in \cal X,
		\end{equation}
		is satisfied $\Po$-almost surely.
		\item The parameter space $\cal X$ is compact, every element of $\cal H$ is $\Po$-integrable, every element $\bm x \mapsto h(\bm x, \rv \xi)$ in $\cal H$ is continuous on $\cal X$ $\Po$-almost-surely, and there exists a $\Po$-integrable envelop $m(\rv \xi)$ such that
		\begin{equation}\label{eq:H-envelop}
			\sup_{\bm x \in \cal X} |h(\bm x, \rv \xi)| \le m(\rv \xi)
		\end{equation}
		is satisfied $\Po$-almost surely.
		\item Every element in $\cal H$ is a finite linear combination of other $\Po$-integrable functions; that is,
		\begin{equation}\label{eq:H-linear-combination}
			\cal H \defeq \left\{\left.\sum^l_{i=1} x_i f_i(\rv \xi) \right| 
            \begin{array}{l}
            \bm x \in \cal X \subseteq \R^l, ~
            \E_{\Po} f_i(\rv \xi) < \infty 
            \end{array}
            \right\}.
		\end{equation}
		This type of $\cal H$ is popular in machine learning, for example, when the hypothesis class $\cal H$ is a well-designed reproducing kernel Hilbert space.
        
		\item The function class $\cal H$ is a Vapnik--Chervonenkis (VC) class; that is, the VC index of $\cal H$ is finite. For example, the function class in \eqref{eq:H-linear-combination} is a VC class.
	\end{enumerate}
 
Condition C4) holds if, for example, one of the following is satisfied:
	\begin{enumerate}[a)]
		\item For any $\bm x \in \cal X'$, if the function $h(\bm x, \rv \xi)$ is continuous at $\bm x$, $\Po$-almost surely, and the function $h(\bm x_n, \rv \xi)$ is dominated by a $\Po$-integrable envelop function $m(\rv \xi)$ for every $n$, then $v(\bm x)$ is continuous on $\cal X'$. This is by the dominated convergence theorem.
		\item The fact that $\bm x_n \to \bm x$ implies $v(\bm x_n) \to v(\bm x)$, for every $\bm x_n, \bm x \in \cal X'$. This means that $v(\bm x)$ is continuous on $\cal X'$.
		\item The function $h(\bm x, \rv \xi)$ is convex in $\bm x$, $\Po$-almost surely, so that $v(\bm x)$ is convex and therefore continuous in the interior of $\cal X'$. Note that the convexity of $v(\bm x)$ on $\cal X'$ implies its continuity in the interior of $\cal X'$.
		\item The function $h(\bm x, \rv \xi)$ is strictly convex (resp. strongly convex) in $\bm x$, $\Po$-almost surely, so that $v(\bm x)$ is strictly convex (resp. strongly convex). This means that $v(\bm x)$ is has a unique global minimizer on $\cal X'$.
	\end{enumerate}
 
Condition C5) holds if, for example, one of the following is satisfied:
	\begin{enumerate}[a)]
		\item The function class $\cal H$ is finite and every element of $\cal H$ is $\Po$-square-integrable;
		\item The parameter space $\cal X$ is bounded, every element of $\cal H$ is $\Po$-square-integrable, and there exists a $\Po$-square-integrable function $m(\rv \xi)$ such that
		\[%\begin{equation}\label{eq:H-Lip-donsker}
		|h(\bm x_1, \rv \xi) - h(\bm x_2, \rv \xi)| \le m(\rv \xi) \|\bm x_1 - \bm x_2\|,~~~\forall \bm x_1, \bm x_2 \in \cal X,
		\]%\end{equation}
		is satisfied $\Po$-almost surely.
		\item Every element in $\cal H$ is a finite linear combination of other $\Po$-square-integrable functions; that is,
		\begin{equation}\label{eq:H-linear-combination-2}
			\cal H \defeq \left\{\left.\sum^l_{i=1} x_i f_i(\rv \xi) \right| 
            \begin{array}{l}
            \bm x \in \cal X \subseteq \R^l, ~
            \E_{\Po} [f_i(\rv \xi)]^2 < \infty
            \end{array}
            \right\}.
		\end{equation}
        This type of $\cal H$ is popular in machine learning, for example, when the hypothesis class $\cal H$ is a well-designed reproducing kernel Hilbert space.
		\item The function class $\cal H$ is a Vapnik--Chervonenkis (VC) class; that is, the VC index of $\cal H$ is finite.
	\end{enumerate}
 
Condition C6) holds if, for example, one of the following is satisfied:
	\begin{enumerate}[a)]
		\item There exists a $\Po$-square-integrable function $m(\rv \xi)$ such that
		\[%\begin{equation}\label{eq:H-Lip-donsker}
		|h(\bm x_1, \rv \xi) - h(\bm x_2, \rv \xi)| \le m(\rv \xi) \|\bm x_1 - \bm x_2\|,~~~\forall \bm x_1, \bm x_2 \in \cal X,
		\]%\end{equation}
		is satisfied $\Po$-almost surely.
		
		\item Every element in $\cal H$ is a finite linear combination of other $\Po$-square-integrable functions; that is, \eqref{eq:H-linear-combination-2}. 
        This is because 
		%\[
        $%\begin{array}{ll}
		\E_{\Po}[\sum^l_{i=1} (\hat x_{n,i} - x_{0,i}) f_i(\rv \xi)]^2 
         \le  \sum^l_{i=1} (\hat x_{n,i} - x_{0,i})^2 \cdot \sum^l_{i=1} \E_{\Po}[f_i(\rv \xi)]^2 
         \pcvg 0,
        $ %\end{array}
		%\]
		as $\bmh x_{n} \pcvg \bm x_{0}$.
	\end{enumerate}
 
Therefore, if we assume the pointwise $m(\rv \xi)-$Lipschitz continuity of the function $h(\bm x, \rv \xi)$ where $m(\rv \xi)$ is $\Po$-square-integrable, then Conditions C3-C6 in Theorem \ref{thm:asym-properties} are simultaneously satisfied. In addition, if $\cal H$ takes the form as in \eqref{eq:H-linear-combination-2}, then Conditions C3-C6 in Theorem \ref{thm:asym-properties} are simultaneously satisfied as well. This two situations are sufficient for most of practical machine learning hypothesis classes.

\section{Appdices of Section \ref{sec:experiment}}\label{append:experiments}

\subsection{BDR Support Vector Machine}\label{append:experiments-svm}

The BDR SVM classifier is derived in Appendix \ref{append:bdr-form-svm}. 
Additional experimental results of the BDR SVM on the MNIST dataset are shown in Appendix \ref{append:experiments-figures}. 
Experimental tests of the BDR SVM on the UCI datasets \cite{dua2019}, i.e., the Ionosphere dataset, the Breast Cancer dataset, and the Adult dataset are reported in Appendices \ref{append:svm-experiments-ionosphere}, \ref{append:svm-experiments-breast-cancer}, and \ref{append:svm-experiments-adult}, respectively. 
The tuning method of hyperparameters is the same as that used on the MNIST dataset; see Subsection \ref{sec:expt-svm} in the main body of the paper. The average computational times (averaged over 100 independent trials) for the experiments are provided in Table \ref{tab:cpu-time}, which shows that BDR and DRO are computationally comparable.

\begin{table}[!htbp]
\caption{Averaged CPU Times (Unit: Seconds)}
\label{tab:cpu-time}
\begin{center}
\begin{tabular}{lccc}
\hline
               & BDR  & DRO & SAA  \\
\hline
MNIST (4 vs 9)   & 68   & 66  & 7   \\
Ionosphere     & 4.0  & 3.9 & 0.3  \\
Breast Cancer  & 4.2  & 3.5 & 0.4  \\
Adult (a1a)    & 4.0  & 3.8 & 0.3  \\
\hline
\end{tabular}
\end{center}
\end{table}

\subsubsection{BDR Formulation of SVM}\label{append:bdr-form-svm}
The BDR formulation of the SVM classification problem can be solved with a linear program
\begin{equation}\label{eq:bdr-form-svm}
	\begin{array}{cll}
		\displaystyle \min_{\bm x,~\bm \lambda,~\bm s} & \displaystyle \beta_n \epsilon \lambda_0 + \frac{1}{n} \sum^n_{i = 1}  \lambda_i \\
		s.t. & \displaystyle 1 - Y_i \cdot \langle \bm x, \rvu I_i \rangle \leq \lambda_i, & \forall i \in [n], \\
		& \displaystyle 1 + Y_i \cdot \langle \bm x, \rvu I_i \rangle - \kappa\lambda_0 \leq \lambda_i, & \forall i \in [n], \\
		& 0 \le \lambda_i, &\forall i \in \{0\}\cup [n],\\
		& \sum_{j=1}^l s_j \leq \lambda_0,\\
		& x_j \leq s_j, -x_j \leq s_j, 0 \le s_j, & \forall j \in [l],\\
        % & -x_j \leq s_j, & \forall j \in [l],\\
        % & 0 \le s_j, & \forall j \in [l],\\
        & \bm x \in \R^l,~\bm \lambda \in \R^{n+1},~\bm s \in \R^l,
	\end{array} 
\end{equation}
where $n$ is the size of training samples and $\bm \lambda \defeq (\lambda_0,\lambda_1,\ldots,\lambda_n)$. The derivation process is trivial and therefore omitted here. Just note that the dual norm of the $\infty$-norm is the $1$-norm, and in \cite[Eq.~(19)]{shafieezadeh2019regularization} we have $\bm C = \bm 0$ and $\bm d = \bm 0$ (i.e., $\cal X \defeq \R^l$). In this special case, BDR amounts to DRO, where BDR just employs a $\beta_n$-shrunken radius $\beta_n \epsilon$ for the distributional uncertainty ball. However, this motivational relation no longer holds for complicated learning tasks such as BDR deep learning.

\subsubsection{Additional Experiments on The MNIST Dataset}\label{append:experiments-figures}

Experimental results of the average out-of-sample accuracy on the MNIST dataset for 3 vs 8 over 100 independent trials are shown in Fig. \ref{fig:fig2}, while for 1 vs 7 are in Fig. \ref{fig:fig3}. From the two figures, we can see that the conclusions are consistent with those given in the main body of the paper (i.e., Subsection \ref{sec:expt-svm}): For example, BDR is more robust than DRO to the choice of the radius $\epsilon$ of the distributional uncertainty ball.

\begin{figure}[!htbp]
	\centering
	\subfigure[Mean accuracy against $\epsilon$ \& $\kappa$ ]{
		\includegraphics[height=2.8cm]{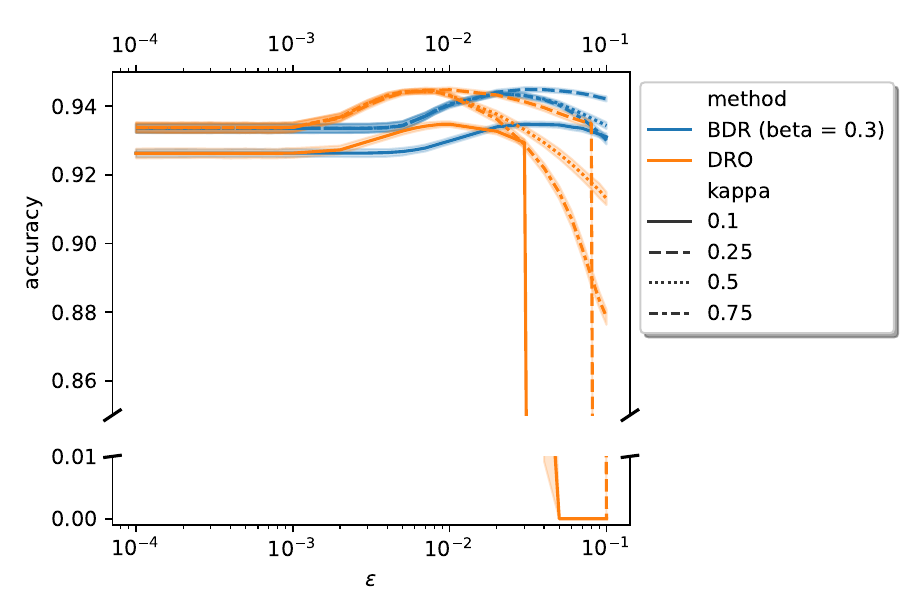}
	}\hspace{-1em}
	\subfigure[Box plot of accuracy]{
		\includegraphics[height=2.8cm]{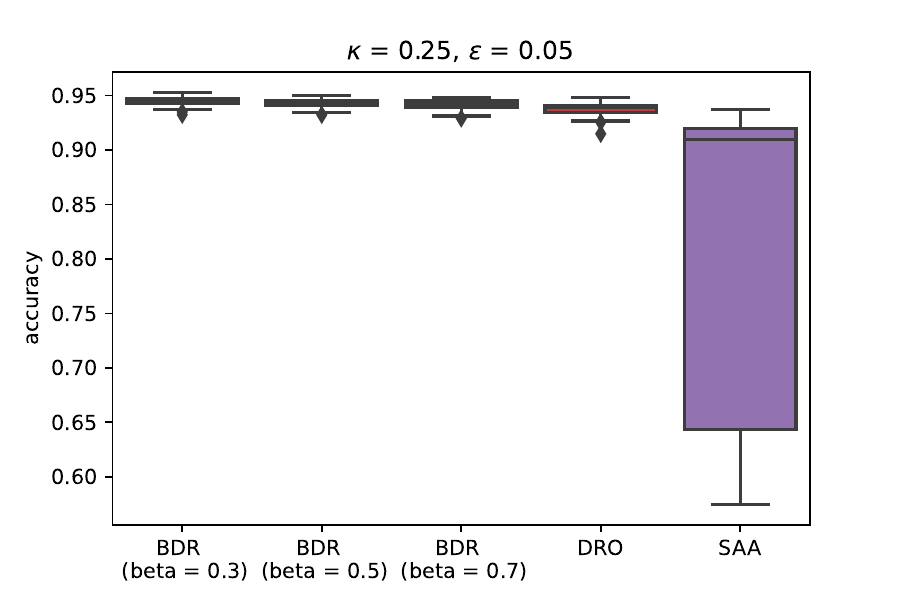}
	}
	\caption{Average out-of-sample accuracy on the MNIST dataset for 3 vs 8 over 100 independent trials.}
	\label{fig:fig2}
\end{figure}

\begin{figure}[!htbp]
	\centering
	\subfigure[Mean accuracy against $\epsilon$ \& $\kappa$  ]{
		\includegraphics[height=2.8cm]{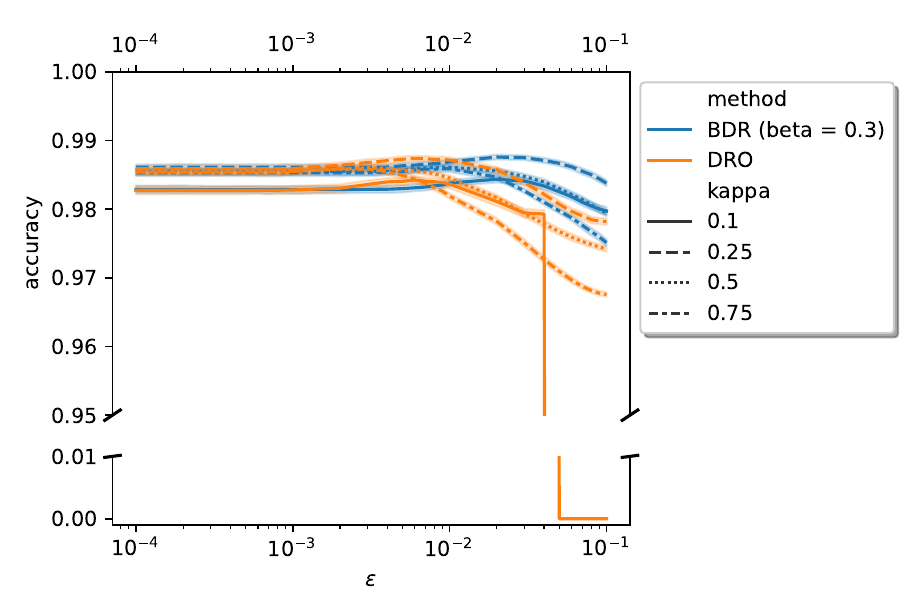}
	}\hspace{-1em}
	\subfigure[Box plot of accuracy]{
		\includegraphics[height=2.8cm]{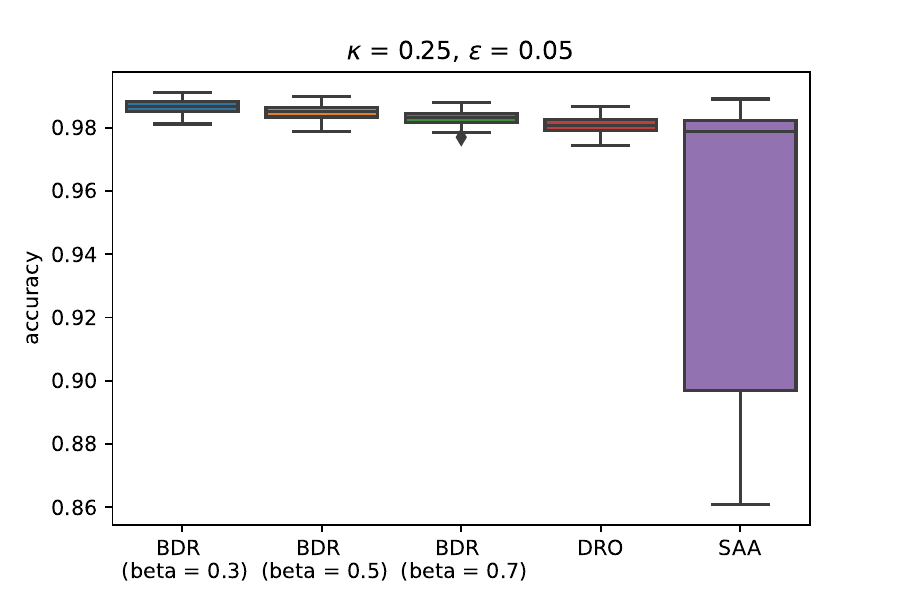}
	}
	\caption{Average out-of-sample accuracy on the MNIST dataset for 1 vs 7 over 100 independent trials.}
	\label{fig:fig3}
\end{figure}

\subsubsection{Experiments on The UCI Ionosphere Dataset}\label{append:svm-experiments-ionosphere}
The results on the UCI Ionosphere dataset are shown in Fig. \ref{fig:Ionosphere}.
\begin{figure}[!htbp]
	\centering
	\subfigure[Mean accuracy against $\epsilon$ \& $\kappa$  ]{
		\includegraphics[height=2.8cm]{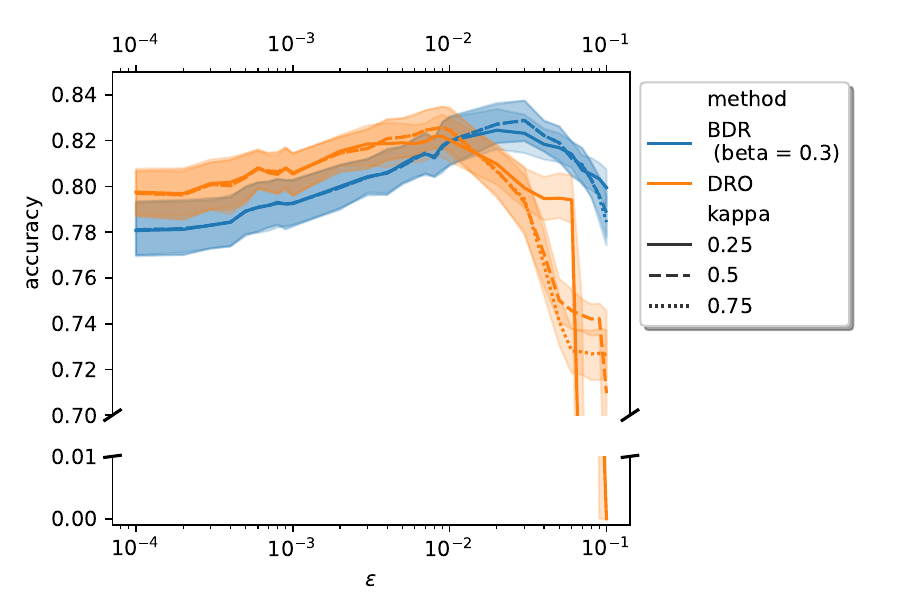}
	}\hspace{-1em}
	\subfigure[Box plot of accuracy]{
		\includegraphics[height=2.8cm]{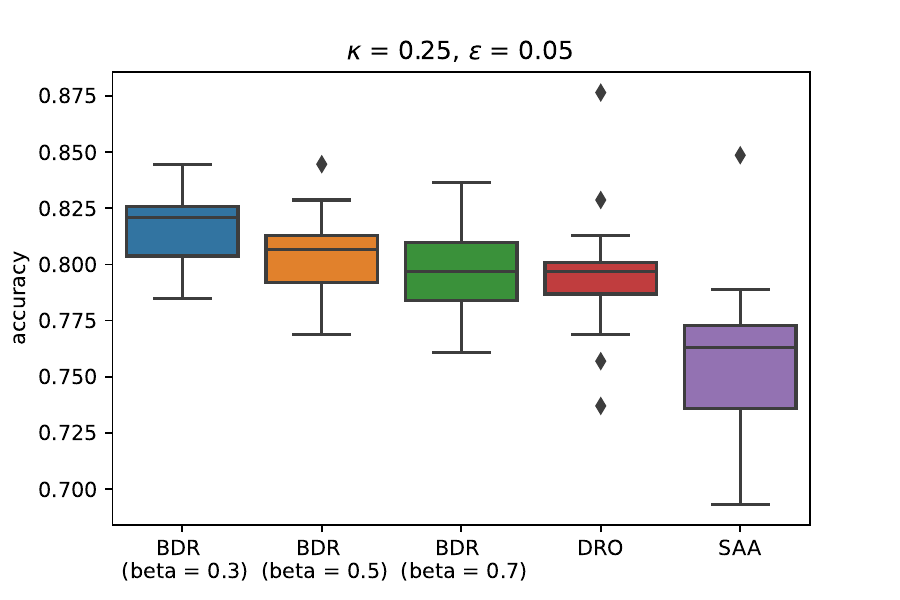}
	}
	\caption{Average out-of-sample accuracy on the UCI Ionosphere dataset over 100 independent trials.}
	\label{fig:Ionosphere}
\end{figure}

\subsubsection{Experiments on The UCI Breast Cancer Dataset}\label{append:svm-experiments-breast-cancer}
The results on the UCI Breast Cancer dataset are shown in Fig. \ref{fig:Breast-Cancer}.
\begin{figure}[!htbp]
	\centering
	\subfigure[Mean accuracy against $\epsilon$ \& $\kappa$  ]{
		\includegraphics[height=2.8cm]{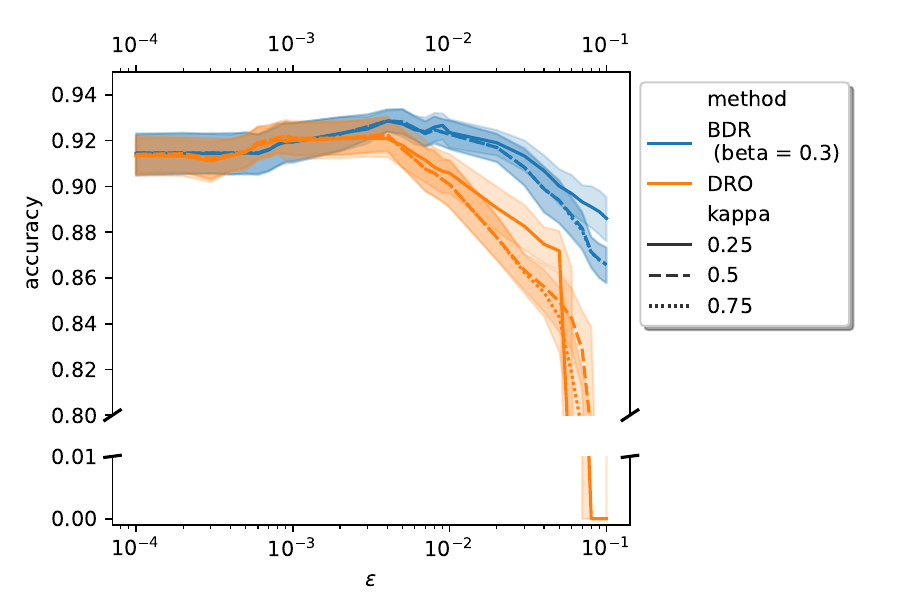}
	}\hspace{-1em}
	\subfigure[Box plot of accuracy]{
		\includegraphics[height=2.8cm]{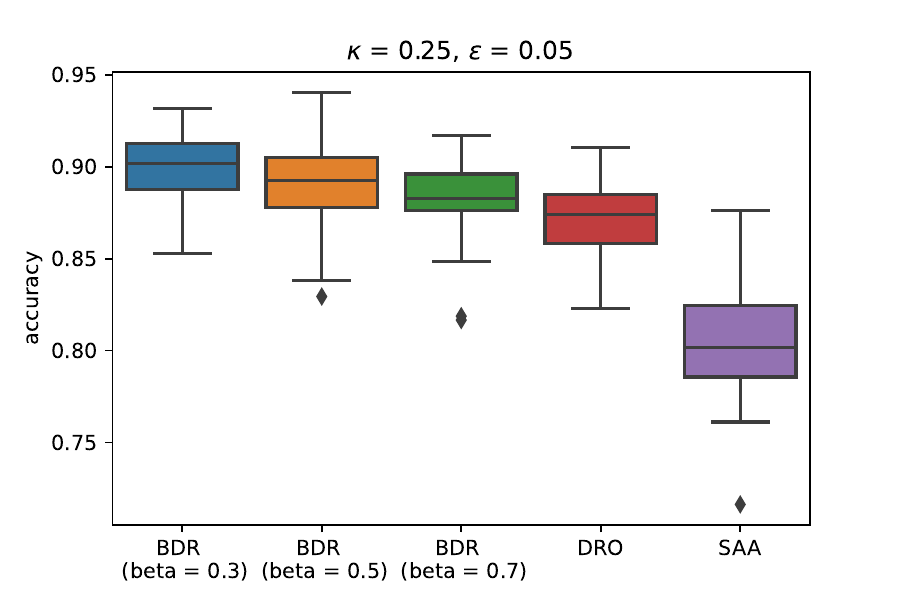}
	}
	\caption{Average out-of-sample accuracy on the UCI Breast Cancer dataset over 100 independent trials.}
	\label{fig:Breast-Cancer}
\end{figure}

\subsubsection{Experiments on The UCI Adult Dataset}\label{append:svm-experiments-adult}
The results on the UCI Adult dataset are shown in Figs. \ref{fig:adult-a1a}-\ref{fig:adult-a5a}.

\begin{figure}[!htbp]
	\centering
	\subfigure[Mean accuracy against $\epsilon$ \& $\kappa$  ]{
		\includegraphics[height=2.8cm]{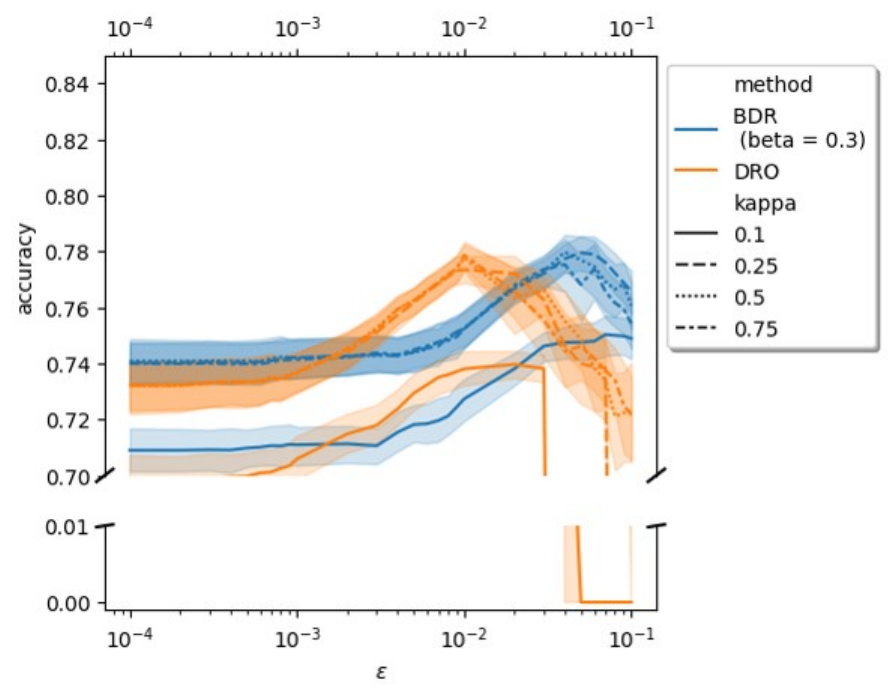}
	}
	\subfigure[Box plot of accuracy]{
		\includegraphics[height=2.8cm]{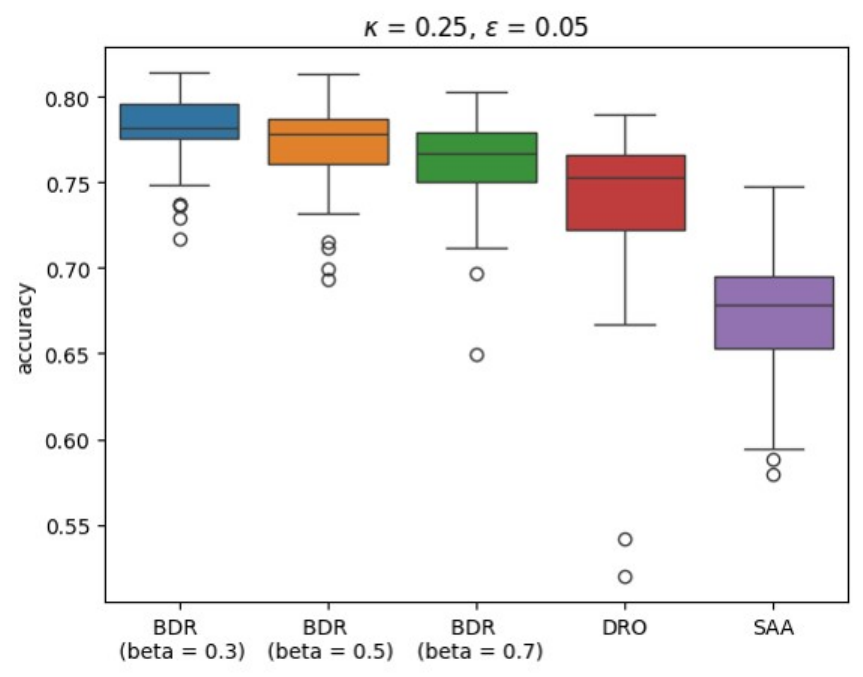}
	}
	\caption{Average out-of-sample accuracy on the UCI Adult dataset (a1a) over 100 independent trials.}
	\label{fig:adult-a1a}
\end{figure}

\begin{figure}[!htbp]
	\centering
	\subfigure[Mean accuracy against $\epsilon$ \& $\kappa$  ]{
		\includegraphics[height=2.8cm]{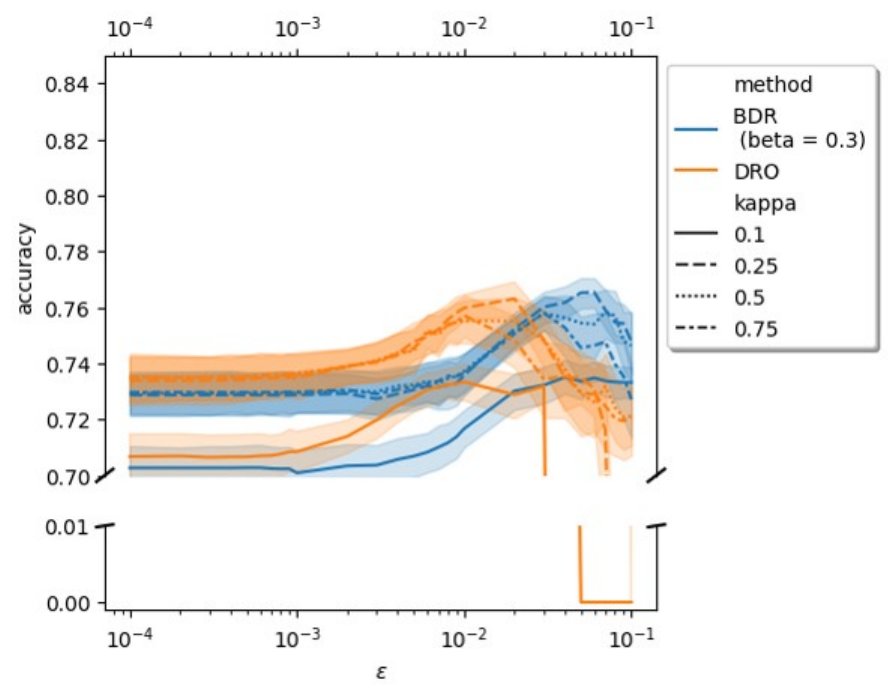}
	}
	\subfigure[Box plot of accuracy]{
		\includegraphics[height=2.8cm]{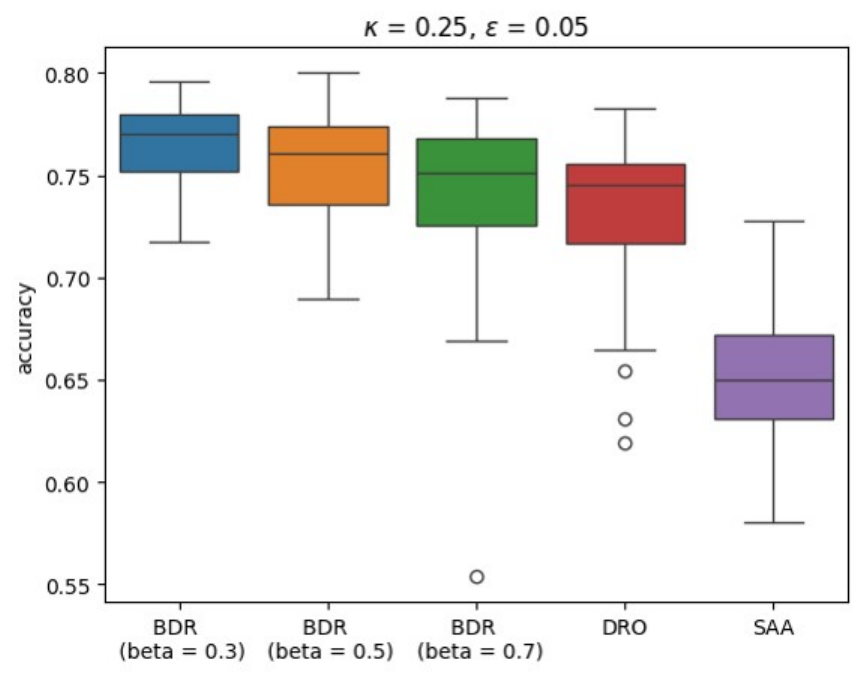}
	}
	\caption{Average out-of-sample accuracy on the UCI Adult dataset (a2a) over 100 independent trials.}
	\label{fig:adult-a2a}
\end{figure}

\begin{figure}[!htbp]
	\centering
	\subfigure[Mean accuracy against $\epsilon$ \& $\kappa$  ]{
		\includegraphics[height=2.8cm]{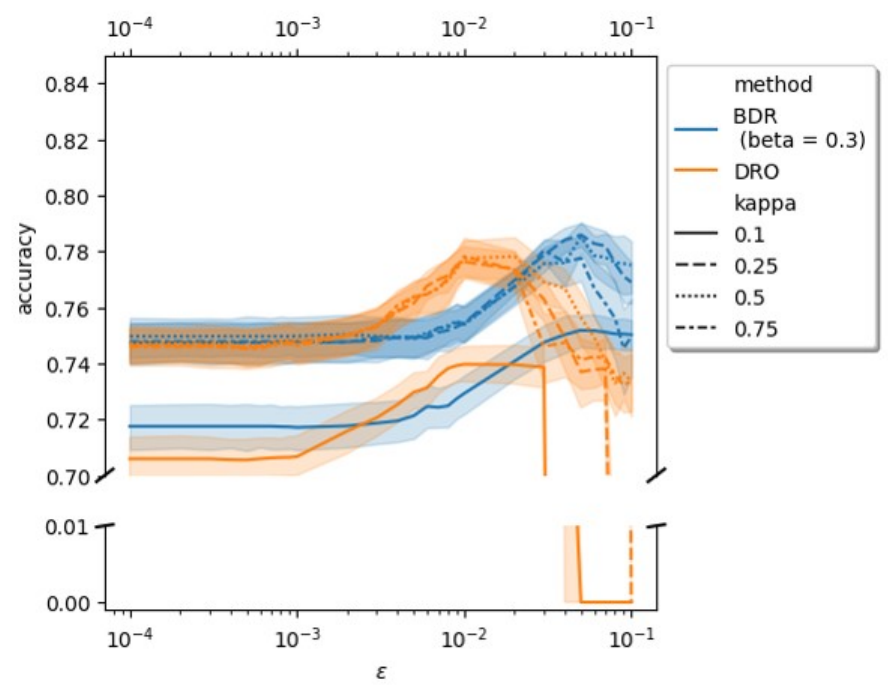}
	}
	\subfigure[Box plot of accuracy]{
		\includegraphics[height=2.8cm]{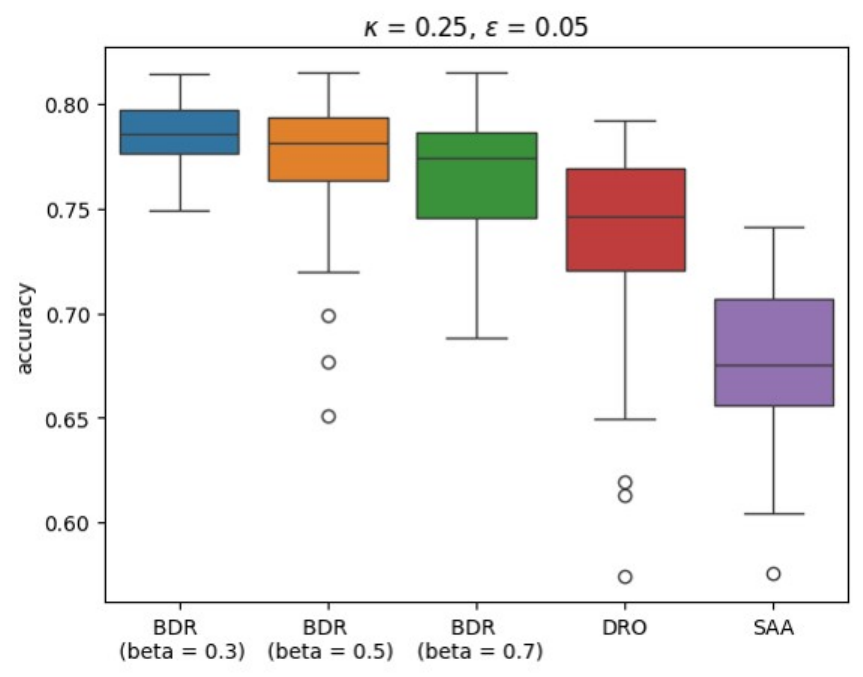}
	}
	\caption{Average out-of-sample accuracy on the UCI Adult dataset (a3a) over 100 independent trials.}
	\label{fig:adult-a3a}
\end{figure}

\begin{figure}[!htbp]
	\centering
	\subfigure[Mean accuracy against $\epsilon$ \& $\kappa$  ]{
		\includegraphics[height=2.8cm]{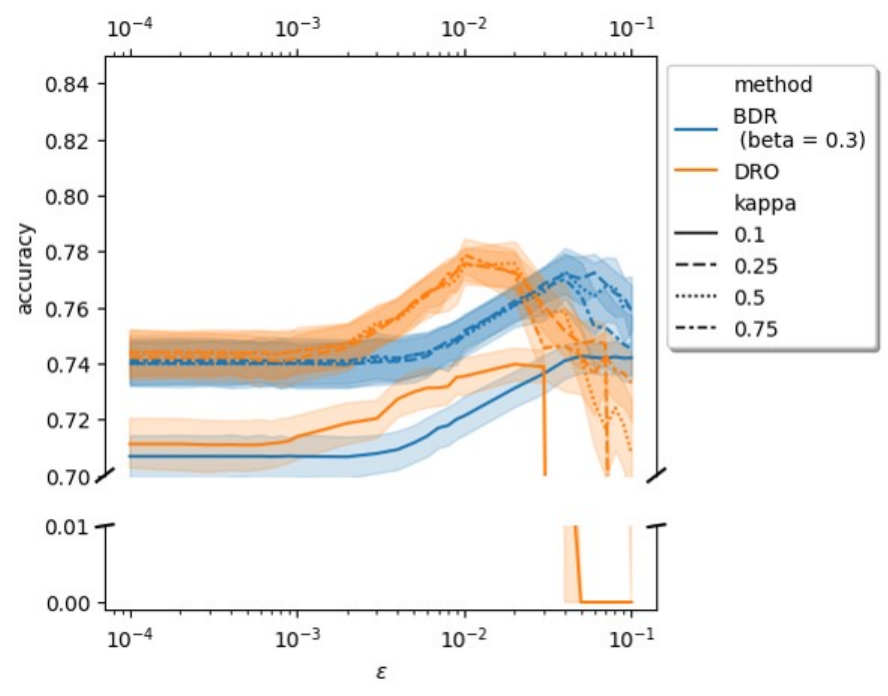}
	}
	\subfigure[Box plot of accuracy]{
		\includegraphics[height=2.8cm]{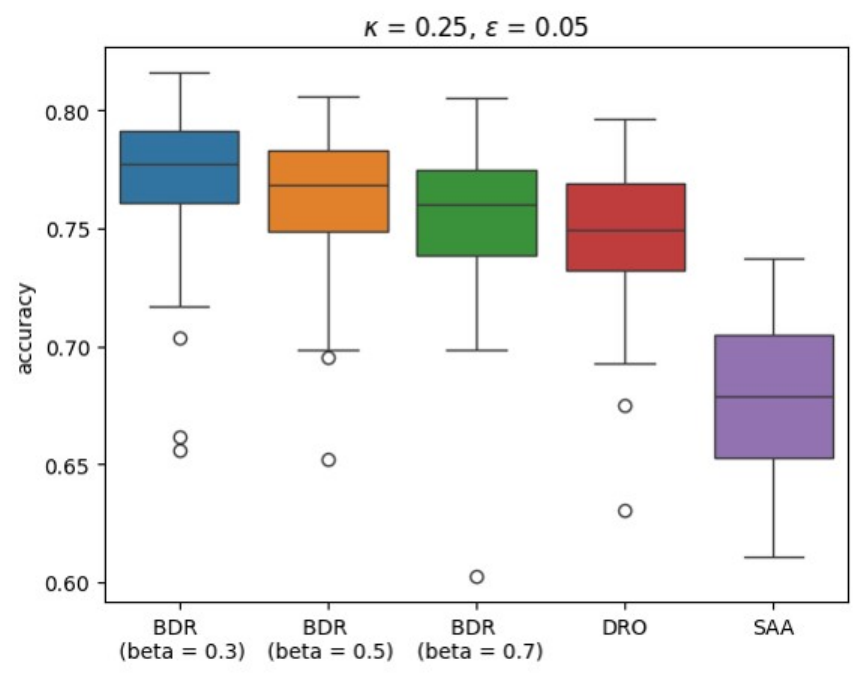}
	}
	\caption{Average out-of-sample accuracy on the UCI Adult dataset (a4a) over 100 independent trials.}
	\label{fig:adult-a4a}
\end{figure}

\begin{figure}[!htbp]
	\centering
	\subfigure[Mean accuracy against $\epsilon$ \& $\kappa$  ]{
		\includegraphics[height=2.8cm]{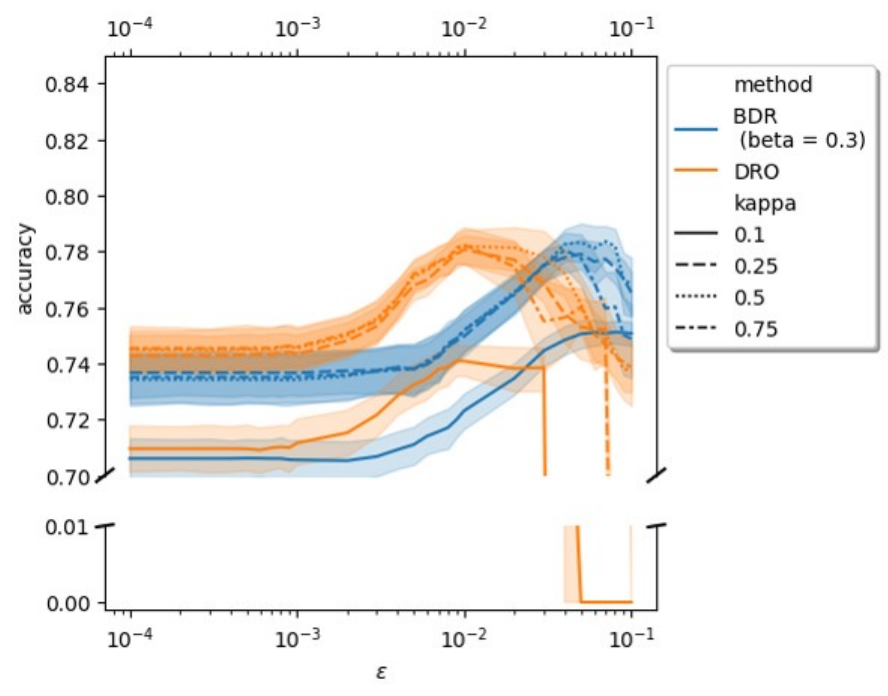}
	}
	\subfigure[Box plot of accuracy]{
		\includegraphics[height=2.8cm]{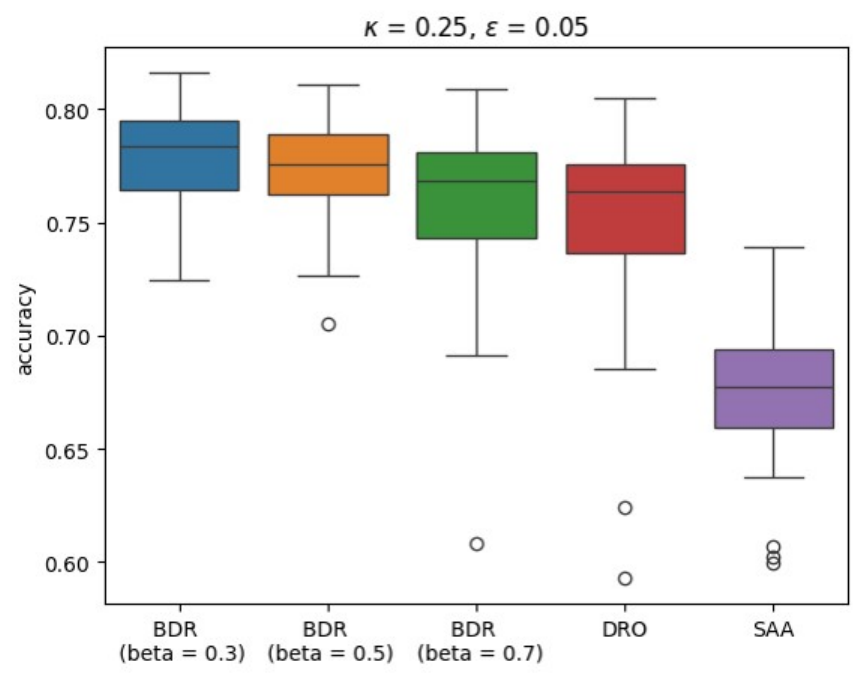}
	}
	\caption{Average out-of-sample accuracy on the UCI Adult dataset (a5a) over 100 independent trials.}
	\label{fig:adult-a5a}
\end{figure}

\subsection{Deep BDR Learning}\label{append:experiments-resnet}
\subsubsection{Dataset Overview}
We provide the numerical details of the utilized datasets in Table~\ref{tab:datasets_detail}.
\begin{table}[!htbp]
\centering
\caption{Summary of Datasets}\label{tab:datasets_detail}
\begin{tabular}{lccc}
\toprule
\textbf{Dataset Name} & \textbf{Train Data Size} & \textbf{Test Data Size} & \textbf{Categories} \\
\midrule\midrule
MNIST        & 60,000          & 10,000         & 10 \\
CIFAR-10     & 50,000          & 10,000         & 10 \\
CIFAR-100    & 50,000          & 10,000         & 100 \\
ModelNet40   & 9,843           & 2,468          & 40 \\
\bottomrule
\end{tabular}

\label{table:datasets}
\end{table}

\subsubsection{Training Details}
All experiments are executed using Python 3.9, PyTorch 1.2, on a NVIDIA TITAN V GPU, ensuring a stable computing environment for deep learning tasks. We also note that in practice we implement the mini-batch version of Algorithm~\ref{algo:grad-descent} which is just similar to the mini-batch SGD.

In 2D image classification, the WideResNet-28 model's training on CIFAR-10 and CIFAR-100 utilizes a batch size of 128 across 200 epochs. The learning rate is initially 0.1, adjusted down to 0.01 at epoch 100 and further to 0.001 at epoch 150. We employ SGD with momentum for optimization, setting weight decay at 0.0005. Furthermore, the training of model on MNIST uses Adam optimizer with learning rate 0.001 without decaying.
There is no extra data augmentation strategy except for the DRO-based adversarial sample construction.

In 3D point cloud classification, we sample 1,024 points of the 2048-point data as the input. PointNet training setup includes a batch size of 32, up to 250 epochs, and an initial learning rate of 0.001, adjusted by a decay mechanism. The Adam optimizer is used for training. The learning rate's decay step is set to 200,000, with a decay rate of 0.7. On the other hand, DGCNN training specifies a batch size of 32, 250 epochs, and a learning rate of 0.1 with SGD (momentum 0.9). It includes a cosine annealing for adaptive learning rate adjustments. There is no extra data augmentation strategy except for the DRO-based adversarial sample construction.

\subsubsection{Adversarial Training Details}
In our study, the PGD attack~\cite{madry2018towards} within a $2$-norm ball is implemented to generate adversarial samples based on DRO cost, of which the parameters are carefully chosen to ensure an effective yet subtle modification of data. 

In the image classification, the epsilon $\epsilon$, defining the maximum perturbation limit per pixel, is set to $0.03$, to maintain the visual similarity of the adversarial images to their originals. The step size $\alpha$, determining the granularity of each update towards the adversarial direction, is chosen as $0.008$. This fine-grained approach allows for precise control over the perturbation process. We iterate this process for $10$ iterations to achieve a balance between perturbation invisibility and the success rate of the attack. In the point cloud classification, we set $\epsilon$ to 0.05, $\alpha$ to 0.01, and the iteration number to 7.

\subsubsection{Search Complexity Analysis}\label{append:complexity} 
Here we provide the quantitative analysis of the time complexity. Let $t_{\text{SAA}}$, $t_{\text{DRO}}$, and $t_{\text{BDR}}$ denote the training time per epoch for SAA, DRO, and BDR method; $t_{\text{BDR}}$ can be divided into two phases: $\beta$-searching phase by cross-validation and BDR training phase, formally, 
\begin{equation} 
t_{\text{BDR}} = t_{\text{Search}} + t_{\text{Train}}. 
\end{equation} 
For a selected $\beta$, we have $t_{\text{Train}} = \beta t_{\text{DRO}} + (1 - \beta) t_{\text{SAA}}$ due to Algorithm \ref{algo:grad-descent}. Suppose the set of candidate $\beta$ is \{$\beta_1,\cdots, \beta_k$\}, we implement the training on each $\beta_i$ as 
\begin{equation} 
t_{\text{Search}} = r \sum_{i=1}^k t_{\text{Train with }\beta_i} = r \sum_{i=1}^k (\beta_i t_{\text{DRO}} + (1 - \beta_i) t_{\text{SAA}}) 
\end{equation} 
where $r\in (0, 1]$ is a factor standing for the effect of early stop for cross-validation. Thus, the total time when $\beta^*$ is selected as the optimal one is \begin{equation} 
\begin{array}{cl}
t_{\text{BDR}} &= r \sum_{i=1}^k t_{\text{Train with }\beta_i} + t_{\text{Train}} \\

&= r \sum_{i=1}^k (\beta_i t_{\text{DRO}} + (1 - \beta_i) t_{\text{SAA}}) + \\

& \quad \beta^* t_{\text{DRO}} + (1 - \beta^*) t_{\text{SAA}} 
\end{array}
\end{equation} 
if we consider the upper bound of total time when $r=1$, we have 
\begin{equation} 
t_{\text{BDR}} \leq (\max_{i} \beta_i + \sum_{i=1}^k \beta_i ) t_{\text{DRO}} + (k + 1 -\sum_{i=1}^k \beta_i- \min_{i} \beta_i) t_{\text{SAA}}. 
\end{equation} 
Considering the search set \{$0.5,0.1, 0.05, 0.01$\}, we have 
\begin{equation} t_{\text{BDR}} \leq 1.16 t_{\text{DRO}} + 3.83 t_{\text{SAA}} = 1.6 t_{\text{DRO}} 
\end{equation} 
The upper bound of $ t_{\text{BDR}}$ is $1.6 t_{\text{DRO}}$ when $t_{\text{DRO}}\backslash t_{\text{SAA}} = 9 $ in practice. However, since the usage of early stop (i.e., $r<1$) and $t_{\text{Train}}<\max_{i} \beta_i t_{\text{DRO}} + (1 - \min_{i} \beta_i) t_{\text{SAA}}$. In practice, we always get $t_{\text{BDR}}\approx t_{\text{DRO}}$, as shown in Table~\ref{tab:method_time}. Table~\ref{tab:method_time} provides the time used per epoch for DRO, SAA, and BDR learning for CIFAR-10 and 50\% data experiment. Our searching time is the equivalent time used per epoch for all $\beta$ validation training.

\begin{table}[!htbp]
\centering\caption{Averaged time used per epoch for various methods in CIFAR-10 and 50\% data experiment}\label{tab:method_time}
\begin{tabular}{ll}
\toprule
Method & Time (s) \\
\midrule
DRO & 888.0 $\pm$ 7.3 \\
SAA & 101.1 $\pm$ 1.5 \\ \midrule 
BDR ($\beta$=0.01) & 111.2 $\pm$ 5.6 \\
BDR ($\beta$=0.05) & 143.3 $\pm$ 12.5 \\
BDR ($\beta$=0.1) & 184.4 $\pm$ 16.5 \\
BDR ($\beta$=0.5) & 539.3 $\pm$ 29.3 \\ \midrule
BDR search time & 694.0 \\
\bottomrule
\end{tabular}
\end{table}

As the estimated $\beta^*$ is 0.05, the total time of the BDR method is 837.3s, which is less than the DRO method. The rationale behind this is that DRO optimization often requires significantly more time than SAA. Algorithm~\ref {algo:grad-descent}, by integrating the two, effectively reduces the overall time required for DRO optimization even though we have to conduct the searching by cross-validation.

\bibliographystyle{IEEEtran}
\bibliography{References}

\newpage

\setcounter{page}{1}

\begin{center}
    \bf Supplementary Materials
\end{center}

\section{Appendices of Section \ref{subsec:properties-BDR}}
\subsection{Proof of Theorem \ref{thm:asym-properties}}\label{append:proof-asym-properties}
\begin{proof}
	For every $\bm x \in \cal X'$ such that the SAA and DRO objectives are bounded in $\Pon$-probability, we have
    {\footnotesize
	\[
    \begin{array}{cl}
	v_{b,n}(\bm x) - v_{n}(\bm x) &= \beta_n \max_{\P \in B_{\epsilon_n}(\Pnh)} \E_{\P} h(\bm x, \rv \xi) + \\
    & \quad \quad (1-\beta_n) \E_{\Pnh} h(\bm x, \rv \xi) - \E_{\Pnh} h(\bm x, \rv \xi) \\
    & \pcvg 0,
    \end{array}
	\]}because $\beta_n \to 0$.
	Hence, by Slutsky's theorem, $v_{b,n}(\bm x)$ shares the same asymptotic properties with $v_n(\bm x)$, for every $\bm x \in \cal X'$. As a result, Statement S1) and S4) are immediate due to the conventional strong law of large numbers, i.e.,
    {\footnotesize
	\[
	v_{n}(\bm x) = \E_{\Pnh} h(\bm x, \rv \xi) \ascvg \E_{\Po}h(\bm x, \rv \xi) = v(\bm x),~~~\forall x \in \cal X',
	\]}and the conventional central limit theorem
    {\footnotesize
	\[
	\sqrt{n}[\E_{\Pnh} h(\bm x, \rv \xi) - \E_{\Po}h(\bm x, \rv \xi)] \dcvg N(0, \D_{\Po}h(\bm x, \rv \xi)), ~~~\forall x \in \cal X',
	\]}respectively.
	
	Suppose the DRO sub-problem is solved by $\Pnb$ such that
    {\footnotesize
	\[
	\E_{\Pnb} h(\bm x, \rv \xi) = \max_{\P \in B_{\epsilon_n}(\Pnh)} \E_{\P} h(\bm x, \rv \xi),
	\]}for $\bm x \in \cal X'$.
	Note that this assumption is reasonable due to Condition C1). We have
    {\footnotesize
	\[
	\begin{array}{l}
		\displaystyle \sup_{\bm x \in \cal X'}|v_{b,n}(\bm x) - v(\bm x)| \\
		\quad \quad = \displaystyle \sup_{\bm x \in \cal X'} \left| \E_{\beta_n \Pnb + (1 - \beta_n) \Pnh} h(\bm x, \rv \xi) - \E_{\Po} h(\bm x, \rv \xi)\right|  \\
		\quad \quad = \displaystyle \sup_{\bm x \in \cal X'}\left|\beta_n \Big[ \E_{\Pnb} h(\bm x, \rv \xi) - \E_{\Po} h(\bm x, \rv \xi) \Big] + \right. \\
        \quad \quad \quad \quad \left. (1 - \beta_n) \Big[ \E_{\Pnh} h(\bm x, \rv \xi) - \E_{\Po} h(\bm x, \rv \xi)\Big]\right|  \\
		\quad \quad \le \displaystyle \beta_n \sup_{\bm x \in \cal X'} \left| \E_{\Pnb} h(\bm x, \rv \xi) - \E_{\Po} h(\bm x, \rv \xi) \right| + \\
        \quad \quad \quad \quad (1 - \beta_n) \sup_{\bm x \in \cal X'} \left| \E_{\Pnh} h(\bm x, \rv \xi) - \E_{\Po} h(\bm x, \rv \xi) \right|  \\
		\quad \quad  \pcvg 0.
	\end{array}
	\]}The first term vanishes because $\beta_n$ approaches zero and $\sup_{\bm x \in \cal X'} \left| \E_{\Pnb} h(\bm x, \rv \xi) - \E_{\Po} h(\bm x, \rv \xi) \right|$ is finite on $\cal X'$, whereas the second term decays because $\cal H$ is $\Po$-Glivenko--Cantelli. As a result, $\min_{\bm x} v_{b, n} \pcvg \min_{\bm x} v(\bm x)$, as $n \to \infty$, because
    {\footnotesize
	\[
    \begin{array}{cl}
	|v_{b, n} (\bmh x_{b, n}) - v(\bm x_0) | &= \displaystyle |\min_{\bm x \in \cal X'} v_{b, n} (\bm x) - \min_{\bm x \in \cal X'} v(\bm x) | \\
    &\le \displaystyle \sup_{\bm x \in \cal X'} |v_{b, n} (\bm x) - v(\bm x) | \\
    &\pcvg 0.
    \end{array}
	\]}This is Statement S2).
	
	For every $\bmh x_{b,n} \in \hat{\cal X}_{b,n}$, we have
    {\footnotesize
	\[
	\begin{array}{l}
		\displaystyle |v(\bmh x_{b,n}) - \min_{\bm x \in \cal X'} v(\bm x)| \\
        \quad \le \displaystyle |v(\bmh x_{b,n}) - v_{b,n}(\bmh x_{b,n})| + \displaystyle |v_{b,n}(\bmh x_{b,n}) - \min_{\bm x \in \cal X'} v(\bm x)| \\
		\quad \le \displaystyle \sup_{\bm x \in \cal X'} |v_{b,n}(\bm x) - v(\bm x)| + \displaystyle |v_{b,n}(\bmh x_{b,n}) - \min_{\bm x \in \cal X'} v(\bm x)| \\
		\quad \pcvg 0.
	\end{array}
	\]}Therefore, due to Condition C4), there exists $\bm x_0 \in \cal X_0$ such that $\bmh x_{b,n} \pcvg \bm x_0$, which proves Statement S3). (One may use a contradiction, by assuming that the limit point of $\bmh x_{b,n}$ is not in $\cal X_0$, to verify this claim.)
	
	By Conditions C5) and C6), we have $\bb G_n h(\bmh x_n, \rv \xi) \dcvg \bb G_{\Po} h(\bm x_0, \rv \xi) \sim N(0, \D_{\Po} h(\bm x_0, \rv \xi))$; see \cite[Lemma 19.24]{vdv1998asymptotic}.
	On the one hand, we have
    {\footnotesize
	\[
	\begin{array}{l}
		\sqrt{n}[v_{b,n}(\bmh x_{b,n}) - v(\bm x_0)] \\
        \quad =  \E_{\sqrt{n}[\beta_n \Pnb + (1 - \beta_n) \Pnh]} h(\bmh x_{b,n}, \rv \xi) - \E_{\sqrt{n}\Po} h(\bm x_0, \rv \xi) \\
		\quad \le \E_{\sqrt{n}[\beta_n \Pnb + (1 - \beta_n) \Pnh]} h(\bm x_0, \rv \xi) - \E_{\sqrt{n}\Po} h(\bm x_0, \rv \xi) \\
		\quad = \sqrt{n}[\E_{\Pnh} h(\bm x_0, \rv \xi) - \E_{\Po}h(\bm x_0, \rv \xi)] + o_p(1) \\
		\quad \dcvg \bb G_{\Po} h(\bm x_0, \rv \xi),
	\end{array}
	\]}where the second equality is because $\sqrt{n} \beta_n \to 0$ and $o_p(1)$ in the second equality denotes the \quotemark{small-Oh} notation (i.e., $a_n = o_p(1)$ implies that the sequence $\{a_n\}$ converges in probability to zero as $n \to \infty$), and the convergence in distribution is due to the fact that $h(\bm x_0, \cdot)$ is in $\cal H$ and $\cal H$ is $\Po$-Donsker. By Slutsky's theorem, it implies that 
	\[
    \begin{array}{l}
	\E_{\sqrt{n}[\beta_n \Pnb + (1 - \beta_n) \Pnh]} h(\bm x_0, \rv \xi) - \E_{\sqrt{n}\Po} h(\bm x_0, \rv \xi) \\
    \quad \dcvg \bb G_{\Po} h(\bm x_0, \rv \xi).
    \end{array}
	\]
	On the other hand, we have 
    {\footnotesize
	\[
	\begin{array}{l}
		\sqrt{n}[v_{b,n}(\bmh x_{b,n}) - v(\bm x_0)] \\
        \quad = \E_{\sqrt{n}[\beta_n \Pnb + (1 - \beta_n) \Pnh]} h(\bmh x_{b,n}, \rv \xi) - \E_{\sqrt{n}\Po} h(\bm x_0, \rv \xi) \\
		\quad \ge \E_{\sqrt{n}[\beta_n \Pnb + (1 - \beta_n) \Pnh]} h(\bmh x_{b,n}, \rv \xi) - \E_{\sqrt{n}\Po} h(\bmh x_{b,n}, \rv \xi) \\
		\quad = \sqrt{n}[\E_{\Pnh} h(\bmh x_{b,n}, \rv \xi) - \E_{\Po}h(\bmh x_{b,n}, \rv \xi)] + o_p(1) \\
		\quad \dcvg \bb G_{\Po} h(\bmh x_{b,n}, \rv \xi) \\
		\quad = \bb G_{\Po} h(\bm x_0, \rv \xi) + o_p(1),
	\end{array}
	\]}where the second equality is because $\sqrt{n}\beta_n \to 0$, the convergence in distribution is due to the fact that $h(\bmh x_{b,n}, \cdot)$ is in $\cal H$ and $\cal H$ is $\Po$-Donsker, and the third equality is because the function $\bb G_{\Po}h(\cdot, \rv \xi)$ is (uniformly) continuous\footnote{Almost all sample paths $f \mapsto \bb G_{\Po}(f), \forall f \in \cal F$ of the $\Po$-Brownian bridge process $\bb G_{\Po}$ are uniformly continuous on the semi-metric space $(\cal F,d)$ where $d$ is a semi-metric on $\cal F$; see \cite[Lemma 18.15]{vdv1998asymptotic}.} so that the continuous mapping theorem applies. By Slutsky's theorem, it implies that 
    
    {\footnotesize
	\[
    \begin{array}{l}
	\E_{\sqrt{n}[\beta_n \Pnb + (1 - \beta_n) \Pnh]} h(\bmh x_{b,n}, \rv \xi) - \E_{\sqrt{n}\Po} h(\bmh x_{b,n}, \rv \xi) \\
    \quad \quad \dcvg \bb G_{\Po} h(\bm x_0, \rv \xi).
    \end{array}
	\]
	}Therefore, by the squeeze theorem, we have
    {\footnotesize
	\[
	\sqrt{n}[v_{b,n}(\bmh x_{b,n}) - v(\bm x_0)] \dcvg \bb G_{\Po} h(\bm x_0, \rv \xi) \sim N(0, \D_{\Po} h(\bm x_0, \rv \xi)),
	\]
    }because the cumulative distribution function of $N(0, \D_{\Po} h(\bm x_0, \rv \xi))$ is continuous everywhere on $\R$.
	This completes the proof.
\end{proof}

\subsection{Asymptotic Normality of the Optimal Solution}\label{append:asymptotic-normality-optimal-solution}
The asymptotic normality of the optimal solution is established below.

\begin{proposition}[{\bf Asymptotic Normality of Optimal Solution}]\label{prop:asy-normality-s}
	For every $\bmh x_{b,n} \in \hat{\cal X}_{b,n}$ and every $\bm x_0 \in \cal X_0$, if Conditions C1) and C2) in Theorem \ref{thm:asym-properties} hold, $\bmh x_{b,n} \pcvg \bm x_0$, the Jacobian $\grad_{\bm x} h(\bm x_0, \rv \xi)$ exists and is $\Po$-square-integrable such that
	\[
	|h(\bm x_1, \rv \xi) - h(\bm x_2, \rv \xi)| \le \grad_{\bm x} h(\bm x_0, \rv \xi) \|\bm x_1 - \bm x_2\|,~~~\forall \bm x_1, \bm x_2 \in \cal X',
	\]
	and the Hessian $\grad^2_{\bm x} h(\bm x_0, \rv \xi)$ exists and is nonsingular and $\Po$-integrable,
	then we have $\sqrt{n}(\bmh x_{b,n} - \bm x_0) \dcvg N(\bm 0, \bm V_{\bm x_0})$ as $n \to \infty$, where
	\[
    \begin{array}{l}
        \bm V_{\bm x_0} \defeq [\E_{\Po} \grad^2_{\bm x} h(\bm x_0, \rv \xi)]^{-1} \cdot \\
        \quad \E_{\Po} [ \grad_{\bm x} h(\bm x_0, \rv \xi) \grad^\top_{\bm x} h(\bm x_0, \rv \xi) ] \cdot [\E_{\Po} \grad^2_{\bm x} h(\bm x_0, \rv \xi)]^{-\top}.
    \end{array}
	 \tag*{$\square$}
	\]
\end{proposition}
\begin{proof}
	The proof is routine in light of proofs of \cite[Thm.~5.23]{vdv1998asymptotic} and Theorem \ref{thm:asym-properties}, and thus, omitted. Just note that a $\Po$-square-integrable function is bounded in $\Po$-probability.
\end{proof}

\subsection{Proof of Theorem \captext{\ref{thm:gen-err}}}\label{append:proof-gen-err}
\begin{proof}
	For every given $\bm x$, if $v_n(\bm x) \ge v(\bm x)$, the first inequality holds for all $\beta_{n, \bm x} \in [0, 1]$ because $v_{r, n}(\bm x) \ge v(\bm x)$ and $v_{r, n}(\bm x) \ge v_n(\bm x)$; note that $\beta_{n, \bm x}$ depends on $\bm x$; if $v_n(\bm x) < v(\bm x)$, the first inequality holds for some $\beta_{n, \bm x} \in [0, 1]$. Therefore, for every $\bm x$, there exists $\beta_{n, \bm x} \in [0, 1]$ such that the inequality 
    {\footnotesize
    $$
    v(\bm x) \le \beta_{n, \bm x} v_{r,n}(\bm x) + (1 - \beta_{n, \bm x}) v_n(\bm x)
    $$
    }holds $\bm x$-point-wisely. Let $\beta^*_{n, \bm x}$ denote the smallest value of $\beta_{n, \bm x}$ that satisfies the above display.
    Because $v_n(\bm x) \le v_{r,n}(\bm x)$, by letting $\beta_n \ge \beta^*_n \defeq \max_{\bm x} \beta^*_{n, \bm x}$, the inequality 
    {\footnotesize $$
    v(\bm x) \le \beta_n v_{r,n}(\bm x) + (1 - \beta_n) v_n(\bm x)
    $$ 
    }holds uniformly for all $\bm x$; note that 
    {\footnotesize $$
    \beta^*_{n, \bm x} v_{r,n}(\bm x) + (1 - \beta^*_{n, \bm x}) v_n(\bm x) \le \beta_{n} v_{r,n}(\bm x) + (1 - \beta_{n}) v_n(\bm x).
    $$
    }Since
    {\footnotesize
    \[
        \beta^*_{n, \bm x} = \frac{v(\bm x) - v_n(\bm x)}{v_{r,n}(\bm x) - v_n(\bm x)},~~~\forall \bm x,
    \]
    }$\beta^*_n$ equals the largest value of the right-hand side of the above display. This completes the proof.
\end{proof}

\subsection{Proof of Theorem \ref{thm:unbiasedness}}\label{append:proof-unbiasedness}

\begin{proof}
	For the DRO problem, if $\Po \in B_{\epsilon_n}(\Pnh)$, as is the case in \eqref{eq:wasserstein-concentration}, we have
    {\footnotesize
	\[
    \begin{array}{cl}
    \min_{\bm x} \E_{\Po} h(\bm x, \rv \xi) 
     &\le \E_{\Po} h(\bmh x_{r, n}, \rv \xi) \\
     &\le \max_{\P \in B_{\epsilon_n}(\Pnh)} \E_{\P} h(\bmh x_{r, n}, \rv \xi) \\
     &= \min_{\bm x} \max_{\P \in B_{\epsilon_n}(\Pnh)} \E_{\P} h(\bm x, \rv \xi).
    \end{array}
	\]
	}The above display implies that
	$
	\min_{\bm x} \E_{\Po} h(\bm x, \rv \xi) \le  \E_{\Pon} \Big[\min_{\bm x} \max_{\P \in B_{\epsilon_n}(\Pnh)} \E_{\P} h(\bm x, \rv \xi)\Big].
	$
	Therefore, the DRO model $\min_{\bm x} \max_{\P \in B_{\epsilon_n}(\Pnh)} \E_{\P} h(\bm x, \rv \xi)$ is always a positively biased estimator of $\min_{\bm x} \E_{\Po} h(\bm x, \rv \xi)$, for every $n$ such that $\Po \in B_{\epsilon_n}(\Pnh)$. On the other hand,
	$
	\E_{\Pon} \left[\min_{\bm x} \E_{\Pnh} h(\bm x, \rv \xi)\right] \le \min_{\bm x} \E_{\Po} h(\bm x, \rv \xi), 
	$
	that is, the SAA model $\min_{\bm x} \E_{\Pnh} h(\bm x, \rv \xi)$ is always a negatively biased estimator of $\min_{\bm x} \E_{\Po} h(\bm x, \rv \xi)$, for every $n$.
	
	As for the BDR model, we have
	{\footnotesize\[
	\begin{array}{ll}
		\displaystyle \min_{\bm x} \Big[\beta_n \max_{\P \in B_{\epsilon_n}(\Pnh)} \E_{\P} h(\bm x, \rv \xi) + (1 - \beta_n) \E_{\Pnh} h(\bm x, \rv \xi)\Big]\\
		
		\quad \le \beta_n \displaystyle \max_{\P \in B_{\epsilon_n}(\Pnh)} \E_{\P} h(\bm x, \rv \xi) + (1 - \beta_n) \E_{\Pnh} h(\bm x, \rv \xi) \\
		
		\quad \le \beta_n \displaystyle \max_{\P \in B_{\epsilon_n}(\Pnh)} \E_{\P} h(\bm x, \rv \xi) + (1 - \beta_n) \displaystyle \max_{\P \in B_{\epsilon_n}(\Pnh)} \E_{\P} h(\bm x, \rv \xi) \\
		
		\quad = \displaystyle \max_{\P \in B_{\epsilon_n}(\Pnh)} \E_{\P} h(\bm x, \rv \xi),
	\end{array}
	\]
	}and therefore,
    
	{\footnotesize\[
    \begin{array}{l}
	\displaystyle \min_{\bm x} \Big[\beta_n \max_{\P \in B_{\epsilon_n}(\Pnh)} \E_{\P} h(\bm x, \rv \xi) + (1 - \beta_n) \E_{\Pnh} h(\bm x, \rv \xi)\Big]  \\
	\quad \le \displaystyle \min_{\bm x} \max_{\P \in B_{\epsilon_n}(\Pnh)} \E_{\P} h(\bm x, \rv \xi).
    \end{array}
	\]}The above implies that, for every $n$, the BDR model gives a smaller estimate than the DRO model. (Since the DRO model is always positively biased, this is a desired property of the BDR model.)
	Furthermore, this means that
    
	{\footnotesize\[
	\begin{array}{l}
		\E_{\Pon} \left[ \displaystyle \min_{\bm x} \Big[\beta_n \max_{\P \in B_{\epsilon_n}(\Pnh)} \E_{\P} h(\bm x, \rv \xi) + (1 - \beta_n) \E_{\Pnh} h(\bm x, \rv \xi)\Big] \right] \\
		\quad \le \displaystyle \E_{\Pon} \left[\displaystyle \min_{\bm x} \max_{\P \in B_{\epsilon_n}(\Pnh)} \E_{\P} h(\bm x, \rv \xi) \right],
	\end{array}
	\]
	}that is, the BDR model tends to have a smaller bias than the DRO model.
	In addition, 
    
    {\footnotesize\[
	\begin{array}{ll}
		\displaystyle \min_{\bm x} \Big[\beta_n \max_{\P \in B_{\epsilon_n}(\Pnh)} \E_{\P} h(\bm x, \rv \xi) + (1 - \beta_n) \E_{\Pnh} h(\bm x, \rv \xi)\Big]\\
		
		\quad \ge \displaystyle \min_{\bm x} \Big[\beta_n \E_{\Pnh} h(\bm x, \rv \xi) + (1 - \beta_n) \E_{\Pnh} h(\bm x, \rv \xi) \Big]\\
		
		\quad = \displaystyle \min_{\bm x} \E_{\Pnh} h(\bm x, \rv \xi).\\
	\end{array}
	\]
	}Hence, for every $n$, the BDR model gives a larger estimate than the SAA model. (Since the SAA model is always negatively biased, this is also a desired property of the BDR model.) Furthermore, this means that
    
	{\footnotesize\[
	\begin{array}{l}
		\E_{\Pon} \left[ \displaystyle \min_{\bm x} \Big[\beta_n \max_{\P \in B_{\epsilon_n}(\Pnh)} \E_{\P} h(\bm x, \rv \xi) + (1 - \beta_n) \E_{\Pnh} h(\bm x, \rv \xi)\Big] \right] \\
		\quad \ge \displaystyle \E_{\Pon} \left[\displaystyle \min_{\bm x} \E_{\Pnh} h(\bm x, \rv \xi) \right],
	\end{array}
	\]
	}that is, the BDR model tends to have a smaller bias than the SAA model.
	
	Since for every $\bm x$,
	{\footnotesize\[
    \begin{array}{cl}
    	\E_{\Pon} \Big[ \min_{\bm x} \E_{\Pnh} h(\bm x, \rv \xi) \Big] &\le 
    	\min_{\bm x}\E_{\Po} h(\bm x, \rv \xi) \\
        &\le 
    	\E_{\Pon} \Big[ \min_{\bm x} \max_{\P \in B_{\epsilon_n}(\Pnh)} \E_{\P} h(\bm x, \rv \xi)\Big],
    \end{array}
	\]
	}there exists $\overline{\beta}_{n} \in [0, 1]$ such that 
	{\footnotesize\[
	\begin{array}{cl}
		\displaystyle \min_{\bm x}\E_{\Po} h(\bm x, \rv \xi) &= \overline{\beta}_n \cdot \E_{\Pon} \displaystyle \Big[\min_{\bm x} \max_{\P \in B_{\epsilon_n}(\Pnh)} \E_{\P} h(\bm x, \rv \xi)\Big] + \\
        &\quad \quad \quad (1 - \overline{\beta}_n) \cdot \E_{\Pon} \Big[\displaystyle \min_{\bm x} \E_{\Pnh} h(\bm x, \rv \xi) \Big] \\
        
		&\le \E_{\Pon} \left[\displaystyle \min_{\bm x} \Big[\overline{\beta}_n \displaystyle \max_{\P \in B_{\epsilon_n}(\Pnh)} \E_{\P} h(\bm x, \rv \xi) + \right. \\ 
        &\quad \quad \quad \left.(1 - \overline{\beta}_n) \E_{\Pnh} h(\bm x, \rv \xi) \Big] \right].
	\end{array}
	\]
	}On the other hand, we have
	{\footnotesize\[
	\begin{array}{l}
		\E_{\Pon} \left[\displaystyle \min_{\bm x} \Big[0 \cdot \displaystyle \max_{\P \in B_{\epsilon_n}(\Pnh)} \E_{\P} h(\bm x, \rv \xi) + (1 - 0) \cdot \E_{\Pnh} h(\bm x, \rv \xi) \Big] \right] \\
		\quad = \E_{\Pon} \Big[ \min_{\bm x} \E_{\Pnh} h(\bm x, \rv \xi) \Big] \\
		\quad \le
		\displaystyle \min_{\bm x}\E_{\Po} h(\bm x, \rv \xi).
	\end{array}
	\]
	}Therefore, there exists $\beta_n \in [0, \overline{\beta}_n]$ such that 
	{\footnotesize\[
    \begin{array}{l}
	\displaystyle \min_{\bm x}\E_{\Po} h(\bm x, \rv \xi)
	= \\
    \quad \E_{\Pon} \left[\displaystyle \min_{\bm x} \Big[{\beta}_n \displaystyle \max_{\P \in B_{\epsilon_n}(\Pnh)} \E_{\P} h(\bm x, \rv \xi) + (1 - {\beta}_n) \E_{\Pnh} h(\bm x, \rv \xi) \Big] \right],
    \end{array}
	\]
	}that is, the BDR model is an unbiased estimator of $\min_{\bm x} \E_{\Po} h(\bm x, \rv \xi)$,
	because the function 
    {\footnotesize
	\[
	\beta \mapsto \E_{\Pon} \left[\displaystyle \min_{\bm x} \Big[{\beta} \cdot \displaystyle \max_{\P \in B_{\epsilon_n}(\Pnh)} \E_{\P} h(\bm x, \rv \xi) + (1 - {\beta}) \cdot \E_{\Pnh} h(\bm x, \rv \xi) \Big] \right]
	\]}is increasing and continuous in $\beta \in [0,1]$.
\end{proof}

\section{Appendices of Section \ref{subsec:solution-method}: Proof of Theorem \ref{thm:DRO-reformulation-wasserstein}}\label{append:DRO-reformulation-wasserstein}

Before we provide the formal proof of Theorem \ref{thm:DRO-reformulation-wasserstein} in Appendix \ref{append:DRO-reformulation-wasserstein-sub}, we prepare with preliminary results in Appendices \ref{subsec:monte-carlo}$\sim$\ref{append:finite-supportness}. The key is to reformulate the infinite-dimensional DRO sub-problem into a finite-dimensional optimization.

\subsection{Monte--Carlo Approximation}\label{subsec:monte-carlo}
In the literature, the DRO problem

{\footnotesize
\begin{equation}\label{eq:dro-method-dist-ball}
	\begin{array}{cc}
		\displaystyle \min_{\bm x \in \cal X} \max_{\P} & \E_{\P} h(\bm x, \rv \xi) \\
		s.t. & \Delta(\P, \Pb) \le \epsilon
	\end{array}    
\end{equation}
}can be reformulated to a non-linear finite-dimensional optimization. For details, see Appendix \ref{subsec:wasserstein-reformulation}.

In this subsection, we propose to use a novel Monte--Carlo-based method to solve \eqref{eq:dro-method-dist-ball}. Suppose 
$%\begin{equation}\label{eq:mm-p}
	\P \approx \textstyle \sum^m_{j=1} \mu_j  \delta_{\rv \zeta_j},
$ %\end{equation}
where $\{\rv \zeta_j\}_{j\in[m]}$ are samples from $\P$,
$\delta_{\rv \zeta_j}$ is the Dirac measure at $\rv \zeta_j$, and the weights $\{\mu_j\}_{j \in [m]}$ can be determined by, e.g., importance sampling through using an appropriate proposal distribution (e.g., uniform distribution).\footnote{C. M. Bishop and N. M. Nasrabadi, Pattern Recognition and Machine Learning, pp. 532. Springer, 2006.} Likewise, we suppose that the set of observations $\{\rv \xi_i\}_{i\in[n]}$ are sampled from $\Pb$ and their weights are $\{\mub_i\}_{i\in[n]}$ and therefore
$%\begin{equation}\label{eq:mm-pb}
	\Pb \approx \textstyle \sum^n_{i=1} \mub_i  \delta_{\rv \xi_i}.
$ %\end{equation}
As a result, all integrals in \eqref{eq:dro-method-dist-ball}, i.e., $\E_{\P} h(\bm x, \rv \xi)$ and those involved in $\Delta$ if any,\footnote{Recall the case where $\Delta$ is the Wasserstein distance defined in \eqref{eq:wasserstein-distance}.} can be approximated by weighted sums; the approximations are exact in the weak convergence sense (i.e., sums converge to integrals) if $\min\{n,m\} \to \infty$ due to the law of large numbers. In practice, we may choose large enough values for $n$ and $m$, which however depends on specific problems. As a result, \eqref{eq:dro-method-dist-ball} transforms to

{\footnotesize
\begin{equation}\label{eq:dro-model-discrete}
	\begin{array}{cl}
		\displaystyle \min_{\bm x} \max_{\{\mu_j, \rv \zeta_j\}_{j\in[m]}} & \sum^m_{j = 1} \mu_j h(\bm x, \rv \zeta_j) \\
		s.t. & \Delta(\P, \Pb) \le \epsilon.
	\end{array}    
\end{equation}}When $\Delta$ is the Wasserstein distance, \eqref{eq:dro-model-discrete} transforms to

{\footnotesize
\begin{equation}\label{eq:dro-model-discrete-wasserstein}
	\begin{array}{cl}
		\displaystyle \min_{\bm x} \max_{
			\bm P, \bm \mu, \{\rv \zeta_j\}}  & \sum^m_{j = 1} \mu_j h(\bm x, \rv \zeta_j) \\
		\text{s.t.} & \sum^n_{i = 1} \sum^m_{j = 1} d^p(\xi_i, \zeta_j) \cdot P_{ij} \le \epsilon^p \\
		& \sum^n_{i = 1} P_{ij} = \mu_j, \quad\quad \forall j \in [m] \\
		& \sum^m_{j = 1} P_{ij} = \mub_i, \quad\quad \forall i \in [n] \\
		& P_{ij} \ge 0, \quad\quad\quad\quad\quad \forall i \in [n], \forall j \in [m],
	\end{array} 
\end{equation}
}where $\bm P \defeq \{P_{ij}\}, \forall i \in [n], \forall j \in [m]$ can be seen as a joint distribution whose marginals are $\bm \mu$ and $\bm \mub$, respectively. By eliminating $\bm \mu$, \eqref{eq:dro-model-discrete-wasserstein} is equivalent to

 {\footnotesize
	\begin{equation}\label{eq:dro-model-discrete-wasserstein-simplified}
		\begin{array}{cll}
			 \displaystyle \min_{\bm x} \displaystyle \max_{
				\bm P, \{\rv \zeta_j\}} &  \sum^m_{j = 1} \sum^n_{i = 1} h(\bm x, \rv \zeta_j) \cdot P_{ij}   \\
			s.t. &  \sum^m_{j = 1} \sum^n_{i = 1} d^p(\xi_i, \zeta_j) \cdot P_{ij} \le \epsilon^p \\
			&  \sum^m_{j = 1} P_{ij} = \mub_i, & \forall i \in [n] \\
			& P_{ij} \ge 0, & \forall i \in [n], \forall j \in [m].
		\end{array} 
	\end{equation}
}

The worst-case distribution solving \eqref{eq:dro-model-discrete-wasserstein-simplified} is given below.
\begin{theorem}\label{thm:finite-supportness}
For every given $\bm x$, the worst-case distribution $\P^*$ solving \eqref{eq:dro-model-discrete-wasserstein-simplified} is supported on at most $n+1$ points in $\Xi$, that is, there exist $\{\mu_j, \rv \zeta_j\}_{j \in [n+1]}$ such that $\P^* = \sum^{n+1}_{j = 1} \mu_j \delta_{\rv \zeta_j}$. Moreover, the discrete worst-case distribution $\P^*$ has the following structure
    
    {\footnotesize
	\begin{equation}\label{eq:worst-case-dist-structure}
		\P^* = \mub_{i_0} \cdot [ q \delta_{\rv \xi_{i_0, 1}} + (1-q) \delta_{\rv \xi_{i_0, 2}}] + \sum^{n}_{j = 1, j \ne i_0} \mub_j \delta_{\rv \zeta_j},
	\end{equation}
	}for one $i_0 \in [n]$, where $0 \le q \le 1$ and $\{\rv \zeta_j\}_{j \in [n+1]} = \{\rv \xi_{i_{0, 1}}\} \bigcup \{\rv \xi_{i_{0, 2}}\} \bigcup \{\rv \xi_i\}_{i \in [n]-i_0}$. To be specific, at most one weight $\mub_{i_0}$ of $\Pb$ is split into two weights of $\P^*$ (N.B.: $q$ is the splitting weight), and the other $n-1$ weights of $\Pb$ (i.e., $\{\mub_i\}_{i \in [n] - i_0}$) are directly inherited by $\P^*$.
\end{theorem}
\begin{proof}
	See Appendix \ref{append:finite-supportness}.
\end{proof}

Theorem \ref{thm:finite-supportness} implies that although $\P^*$ and $\Pb$ have slightly different support sets, $\P^*$ is almost determined by the discrete reference distribution $\Pb$.

\textit{Data-Driven Case}:
If $\Pb \defeq \Pnh = \frac{1}{n} \sum^n_{i=1} \delta_{\rv \xi_i}$ and $\mub_j = 1/n, \forall j \in [n]$, \eqref{eq:dro-model-discrete-wasserstein-simplified} gives
{\footnotesize
\begin{equation}\label{eq:dro-model-discrete-wasserstein-data-driven}
	\begin{array}{lll}
		\displaystyle \min_{\bm x} \max_{
			\bm P, \{\rv \zeta_j\}} & \displaystyle \sum^m_{j = 1} \sum^n_{i = 1} h(\bm x, \rv \zeta_j) \cdot P_{ij}   \\
		s.t. &  \sum^m_{j = 1} \sum^n_{i = 1} d^p(\xi_i, \zeta_j) \cdot P_{ij} \le \epsilon^p \\
		& \sum^m_{j = 1} P_{ij} = \frac{1}{n}, \quad\quad\quad\quad \forall i \in [n] \\
		& P_{ij} \ge 0, \quad\quad\quad \forall i \in [n], \forall j \in [m].
	\end{array} 
\end{equation}
}According to Theorem \ref{thm:finite-supportness}, when conducting the optimization \eqref{eq:dro-model-discrete-wasserstein-data-driven}, it is safe to let $m \defeq n+1$. Note that the solution method \eqref{eq:dro-model-discrete-wasserstein-data-driven} includes several existing duality-based methods as special cases, e.g., Corollary 3.3.1 in \cite{chen2020distributionally}, Section 2.2 in \cite{kuhn2019wasserstein}.

\subsection{Proof of Theorem \ref{thm:finite-supportness}}\label{append:finite-supportness}
\begin{proof}
In \eqref{eq:dro-model-discrete-wasserstein-simplified}, we have $n+1$ constraints, and therefore, at most $n+1$ components in $\bm P$ is non-zero. This further implies that, given $m \ge n+1$, at most $n+1$ components of $\bm \mu$ can be non-zero. In other words, the worst-case distribution solving \eqref{eq:dro-model-discrete-wasserstein-simplified} is supported on at most $n+1$ points for every $m \ge n + 1$. The structure of $\P^*$ is straightforward to be verified by contradiction: If there exist two weights of $\Pb$ to be split, then $\P^*$ needs to be supported on at least $n+2$ points, which contradicts the fact that $\P^*$ is supported on at most $n+1$ points.
\end{proof}

\subsection{Proof of Theorem \ref{thm:DRO-reformulation-wasserstein}}\label{append:DRO-reformulation-wasserstein-sub}
\begin{proof}
	For every $\bm x$, suppose $h(\bm x, \rv \xi)$ is continuous in $\rv \xi$ on $\Xi$. Then, for every $\{\mu_j\}_{j \in [m]}$ and $\{\rv \zeta_j\}_{j \in [m]}$, there exists $\mu'_j = 1/m, \forall j \in [m]$ and $\{\rv \xi'_j\}_{j \in [m]}$ such that
    
    {\footnotesize	
    \begin{equation}\label{eq:weighted-mean-representation}
		\sum^m_{j = 1} \mu_j h(\bm x, \rv \zeta_j) = \sum^m_{j = 1} \frac{1}{m} h(\bm x, \rv \xi'_j).
	\end{equation}
	}This is due to the intermediate value theorem of a continuous function. Hence, using the representation \eqref{eq:weighted-mean-representation} of the weighted sum in the objective of \eqref{eq:dro-model-discrete-wasserstein}, according to Theorem \ref{thm:finite-supportness}, we must have $m = n$ and $\mu_i = \mub_i = 1/n$, for every $i \in [n]$. As a result, \eqref{eq:dro-model-discrete-wasserstein-data-driven} reduces to \eqref{eq:dro-model-discrete-wasserstein-special}. Note that in the interior of $\Xi$, concavity implies continuity. This completes the proof.
\end{proof}

\begin{remark}
The proof process above recovers a well-known reformulation for \eqref{eq:dro-model-discrete-wasserstein-data-driven} in \cite[Cor. 3.3.1]{chen2020distributionally}, which requires the concavity of $h(\bm x, \cdot)$, for every $\bm x$. However, we can relax the concavity to the continuity. 
\stp
\end{remark}

\end{document}